\documentclass{article}

\usepackage{amsmath,amsfonts,bm}

\def\Figref#1{Figure~\ref{#1}}

\def\Secref#1{Section~\ref{#1}}

\def\eqref#1{equation~\ref{#1}}

\def\1{\bm{1}}

\def\vx{{\bm{x}}}

\def\mA{{\bm{A}}}

\def\mC{{\bm{C}}}
\def\mD{{\bm{D}}}
\def\mE{{\bm{E}}}

\def\mI{{\bm{I}}}

\def\mL{{\bm{L}}}

\def\mS{{\bm{S}}}

\def\mU{{\bm{U}}}
\def\mV{{\bm{V}}}
\def\mW{{\bm{W}}}
\def\mX{{\bm{X}}}
\def\mY{{\bm{Y}}}

\DeclareMathAlphabet{\mathsfit}{\encodingdefault}{\sfdefault}{m}{sl}
\SetMathAlphabet{\mathsfit}{bold}{\encodingdefault}{\sfdefault}{bx}{n}

\def\gA{{\mathcal{A}}}

\def\gL{{\mathcal{L}}}

\def\gO{{\mathcal{O}}}

\def\gR{{\mathcal{R}}}

\def\gT{{\mathcal{T}}}

\def\sR{{\mathbb{R}}}
\def\sS{{\mathbb{S}}}

\def\sY{{\mathbb{Y}}}

\usepackage{PRIMEarxiv}

\usepackage[utf8]{inputenc} 
\usepackage[T1]{fontenc}    
\usepackage{hyperref}       
\usepackage{url}            
\usepackage{booktabs}       
\usepackage{amsfonts}       
\usepackage{nicefrac}       
\usepackage{microtype}      
\usepackage{lipsum}
\usepackage{fancyhdr}       
\usepackage{graphicx}       
\usepackage{multicol}
\usepackage{amsthm}
\usepackage{multirow}
\usepackage{wrapfig}
\graphicspath{{media/}}     
\newcommand{\Min}{\textrm{minimize}}

\def\EQREF#1{(\ref{#1})}

\usepackage{verbatim}
\newtheorem{thm}{Theorem}
\newtheorem{cor}{Corollary}

\newtheorem{prop}{Proposition}

\newtheorem{assump}{Assumption}

\title{Adaptive and Implicit Regularization for Matrix Completion
}

\author{
  Zhemin Li\\
  Department of Mathematics\\
  National University of Defense Technology \\
  Changsha 410008,Peoples Republic of China\\
  \texttt{lizhemin@nudt.edu.cn} \\
  \And
  Tao Sun\\
  College of Computer\\
  National University of Defense Technology \\
  Changsha 410008,Peoples Republic of China\\
  \texttt{nudtsuntao@163.com} \\
  \And
  Hongxia Wang\\
  Department of Mathematics\\
  National University of Defense Technology \\
  Changsha 410008,Peoples Republic of China\\
  \texttt{wanghongxia@nudt.edu.cn}, corresponding author \\
  \And
  Bao Wang\\
  Department of Mathematics, Scientific Computing and Imaging Institute\\
  University of Utah, Salt Lake City, UT, 84112, USA\\
  \texttt{wangbaonj@gmail.com} \\
}

\begin{document}
\maketitle

\begin{abstract}
The explicit low-rank regularization, e.g., nuclear norm regularization, has been widely used in imaging sciences. However, it has been found that implicit regularization outperforms explicit ones in various image processing tasks. Another issue is that the fixed explicit regularization limits the applicability to broad images since different images favor different features captured by different explicit regularizations. As such, this paper proposes a new adaptive and implicit low-rank regularization that captures the low-rank prior dynamically from the training data. The core of our new adaptive and implicit low-rank regularization is parameterizing the Laplacian matrix in the Dirichlet energy-based regularization, which we call the regularization \textit{AIR}. Theoretically, we show that the adaptive regularization of AIR enhances the implicit regularization and vanishes at the end of training. We validate AIR's effectiveness on various benchmark tasks, indicating that the AIR is particularly favorable for the scenarios when the missing entries are non-uniform. The code can be found at \href{https://github.com/lizhemin15/AIR-Net}{https://github.com/lizhemin15/AIR-Net}.
\end{abstract}


\section{Introduction}
Natural images usually lie in a high-dimensional ambient space, but with a much lower intrinsic dimension \cite{Tenenbaum2000AGG}, which has motivated many statistical priors for image processing, including sparsity \cite{Cands2007EnhancingSB} and low-rank \cite{Dong2013NonlocalIR}. Other priors, such as Total Variation (TV) and Dirichlet Energy (DE), also play vital roles in image processing. There is no standard way to choose a proper prior to a particular task, and people always depend on empirical experiences.

This paper is devoted to establishing a new adaptive regularization. The core of our proposed adaptive regularization is parameterizing the Laplacian matrix in the conventional DE regularization, and this regularization changes on the fly as the underlying matrix gets updated. In particular, the proposed adaptive regularizer estimates the prior matrix based on the historical information during learning iterations for matrix completion. Compared to the existing regularization schemes for matrix completion, the proposed regularizer is adaptable to matrix completion problems from different applications with different missing patterns. Furthermore, our adaptive regularizer learns hyperparameters from data rather than manually tuning, saving tremendous computational costs. Theoretically, we prove that solving our proposed model with gradient descent can enhance the low-rank property of the recovered matrices, c.f. Theorem \ref{thm..ImplicitReg}, and the gradient descent iterations will result in a non-trivial minimum, c.f. Theorem \ref{thm..ConvReg}. Numerically, we verify our proposed adaptive regularizer's superiority over existing ones on various benchmark matrix completion tasks.

We organize this paper as follows: we first introduce the preliminary for readers unfamiliar with the regularization methods in \Secref{sec..preliminary}.
We will integrate our proposed adaptive regularization scheme with DMF in \Secref{sec..adaptivereg}. Then we illustrate and theoretically prove two implicit regularization properties of our proposed adaptive regularization scheme in \Secref{sec..TheoreticalAna}. We empirically validate the efficacy of our proposed adaptive regularization scheme in \Secref{sec..ExpAna}, followed by concluding remarks. The missing details of technical proofs are provided in the appendix.

\section{Preliminary}
\label{sec..preliminary}
\subsection{Related works}
\subsubsection{Low-rank matrix completion}
Modeling and learning the low-dimensional and low-rank structures of natural images is an important and exciting research area in signal processing and machine learning communities \cite{RemiLam2020MultifidelityDR,Bigoni2021NonlinearDR,Scetbon2021DeepKD,Golts2021DeepET,Khatib2021LearnedGM}.
The low-rank structure is fascinating and widely used prior to image processing. One of the major obstacles in enforcing low-rank prior for image processing is that the exact low-rank minimization is a discrete optimization problem and is NP-hard \cite{fazel2002matrix}. The computational challenge of direct low-rank minimization has motivated several surrogate models of the low-rank regularization. In particular, the authors of \cite{Cands2009ExactMC} relax the low-rank minimization to the nuclear norm minimization, which can be further relaxed to convex optimization with linear constraints. Nevertheless, the convex relaxation solves the low-rank minimization problem while sacrificing the model's accuracy significantly. Indeed, the convex relaxation can introduce unacceptable errors, especially for the problems that are insufficiently low-rank\cite{Boyarski2019SpectralGM}.  
In response, researchers have shifted their attention to the nonconvex models for low-rank guarantees. For instance, the matrix factorization model has been proposed to compensate for the model of nuclear norm relaxation \cite{DanielDLee1999LearningTP,AndrzejCichocki2009NonnegativeMA,AjitPSingh2008AUV,YuXiongWang2013NonnegativeMF}. The matrix factorization model can be formulated as follows: 

Given a matrix $\mX^*\in \sR^{m\times n}$, we seek the decomposition $\mX=\mW^{[0]}\cdot\mW^{[1]}$ and ensure that $\mX\approx \mX^*$, where $\mW^{[0]}\in\sR^{m\times r}$ and $\mW^{[1]}\in\sR^{r\times n}$ with $r\leq \min\{m,n\}$. The decomposition $\mW^{[0]}\cdot\mW^{[1]}$ ensures the rank of $\mX$ to be capped by $r$, and such a decomposition can be numerically computed by alternating the update of $\mW^{[0]}$ and $\mW^{[1]}$. It is straightforward to generalize the above matrix decomposition to the product of multiple matrices (a.k.a. multi-layer matrix factorization), i.e., $\mX = \prod_{\ell=0}^{L-1} \mW^{[\ell]}$\footnote{For the sake of simplicity, we denote the decomposition as $\mX = \prod_{\ell=0}^{L-1} \mW^{[\ell]}$.}, which is called the Deep Matrix Factorization (DMF) \cite{Arora2019ImplicitRI}. Interestingly, the multi-layer matrix factorization enjoys implicit low-rank regularization without any special requirement on the initialization and constraint on $r$, which overcomes the sensitivity on initialization in the two-layer matrix factorization \cite{albright2006algorithms}.

\subsubsection{Implicit regularization: deep matrix factorization}
Due to the advantage of DMF over many existing matrix completion algorithms, various efforts have been contributed to this area. The seminal research about implicit regularization originates from \cite{Gunasekar2018ImplicitRI}, which considers the recovery of a positive semi-definite matrix from symmetric measurements. Based on the shallow matrix factorization model $\mX=\mW^{[0]}\mW^{[1]}$, the authors of \cite{Gunasekar2018ImplicitRI} prove that when $\mW^{[0]},\mW^{[1]}$ are initialized to $\alpha \mI$ with $\alpha>0$, their methods can find the minimal nuclear norm when $\alpha \rightarrow 0$.
In \cite{Arora2019ImplicitRI}, Arora et al. prove a similar result for arbitrary depth and asymmetric cases; however, they ignore the depth effects. As the main contribution of \cite{Arora2019ImplicitRI}, Arora et al. relax the assumptions on the initialization and get rid of the assumption of symmetric measurements. Therefore, we can use more general and convenient initializations that factors in the factorization is balanced at the initialization, i.e., $\left(\mW^{[j+1]}(0)\right)^\top\mW^{[j+1]}(0)=\mW^{[j]}(0)\left(\mW^{[j]}(0)\right)^\top$ for $j=1,\ldots,L-1$. It is worth mentioning that the initialization adopted in \cite{Gunasekar2018ImplicitRI} also satisfy the balanced initialization condition. The theoretical results in \cite{Arora2019ImplicitRI} are established by using the gradient dynamics of DMF. In particular, the singular values of $\mX(t)=\prod_{\ell=0}^{L-1} \mW^{[\ell]}(t)$ evolve with different speeds and thus induce the implicit low-rank regularization. More interestingly, Arora et al. find that the implicit regularization becomes more significant as $L$ increases \cite{Arora2019ImplicitRI}.

Nevertheless, the theoretically-principled initialization methods mentioned above are not popular in practice, instead, Gaussian initialization is preferred for practical usage \cite{He2015DelvingDI}. The bottleneck of the Gaussian initialization is that it
may not yield a low-rank bias starting from the Neural Tangent Kernel (NTK) regime \footnote{This phenomenon happens when DMF is sufficiently wide, and the variance of the initialization is large enough.} \cite{ArthurJacot2018NeuralTK,SimonSDu2018GradientDP,LnacChizat2019OnLT,SanjeevArora2019OnEC,JaehoonLee2019WideNN,JiaoyangHuang2019DynamicsOD}. However, outside the NTK regime, there exists an active regime where the dynamics of DMF are nonlinear and favor the low-rank solution \cite{LnacChizat2018OnTG,GrantMRotskoff2018ParametersAI,SongMei2018AMF,SongMei2019MeanfieldTO,LnacChizat2020ImplicitBO,BlakeWoodworth2020KernelAR,EdwardMoroshko2020ImplicitBI}. Jacot et al. further discuss the effects of initialization for implicit low-rank regularization in \cite{Jacot2021DeepLN}. As the norm of initialized parameters goes to zero, Jacot et al. conjecture that the trajectory of the gradient flow goes from one saddle point to another saddle point, corresponding to matrices with increased ranks. Therefore, we only need to initialize the parameters of DMF to be sufficiently small for the two-layer matrix factorization model. Another advantage of DMF is that we do not need to estimate $r$ for matrix factorization, instead we can simply choose $r=\min\left\{m,n\right\}$. Empirically, DMF outperforms the nuclear norm regularization and the two-layer matrix factorization models for low-rank matrix completion. As shown in Figure~\ref{fig:WhyDMF} (a), DMF achieves remarkable performance for exact low-rank matrix completion problems, and the performance becomes better as the depth of DMF model increases. Moreover, DMF has been applied to many other practical applications, including natural image processing and recommendation systems \cite{Li2020ARD,You2020RobustRV,Boyarski2019SpectralGM}.

\begin{figure}
\begin{center}
\begin{tabular}{ccc}
\hspace{-0.2cm}\includegraphics[width=0.3\columnwidth]{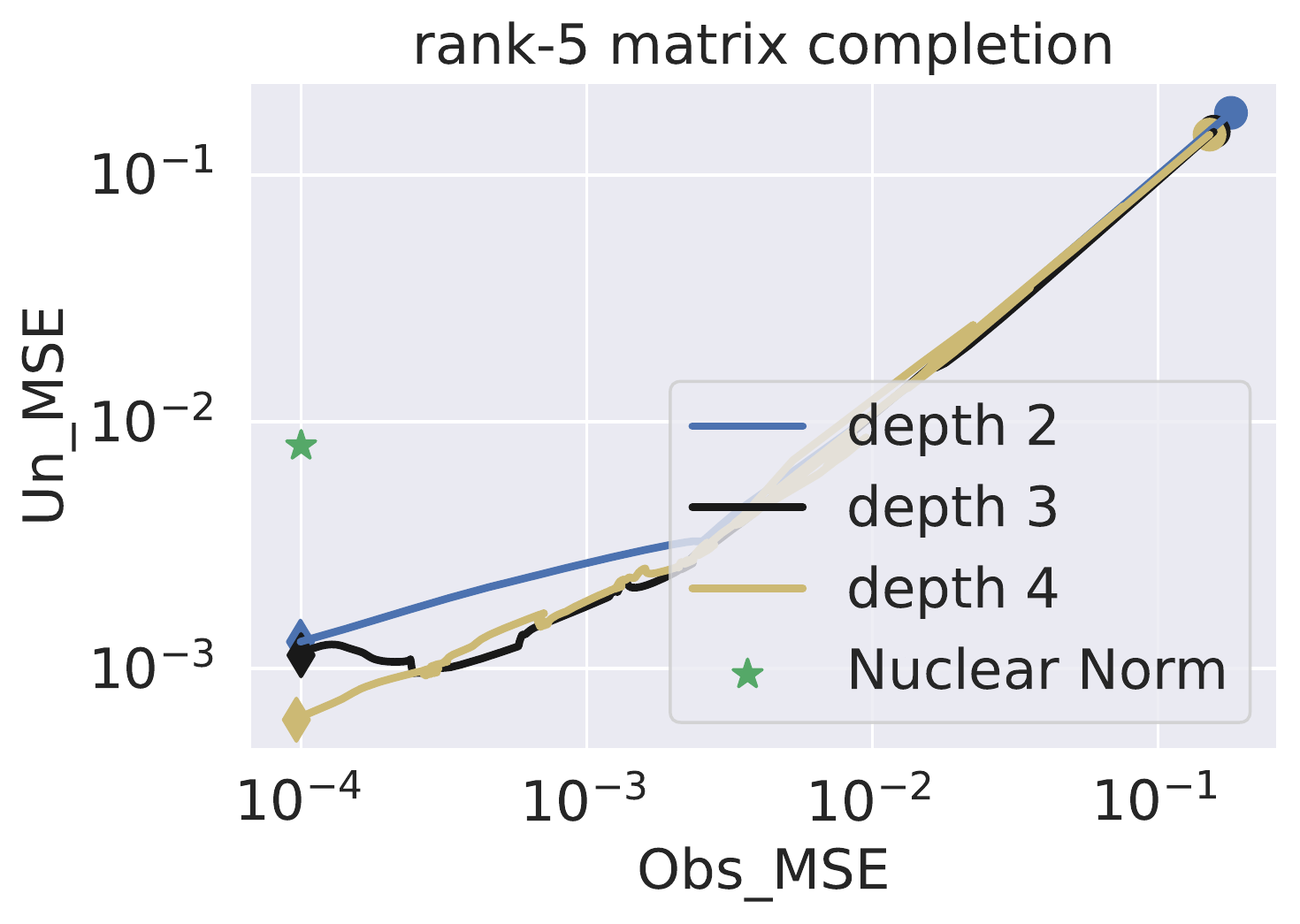}&
\hspace{-0.3cm}\includegraphics[width=0.3\columnwidth]{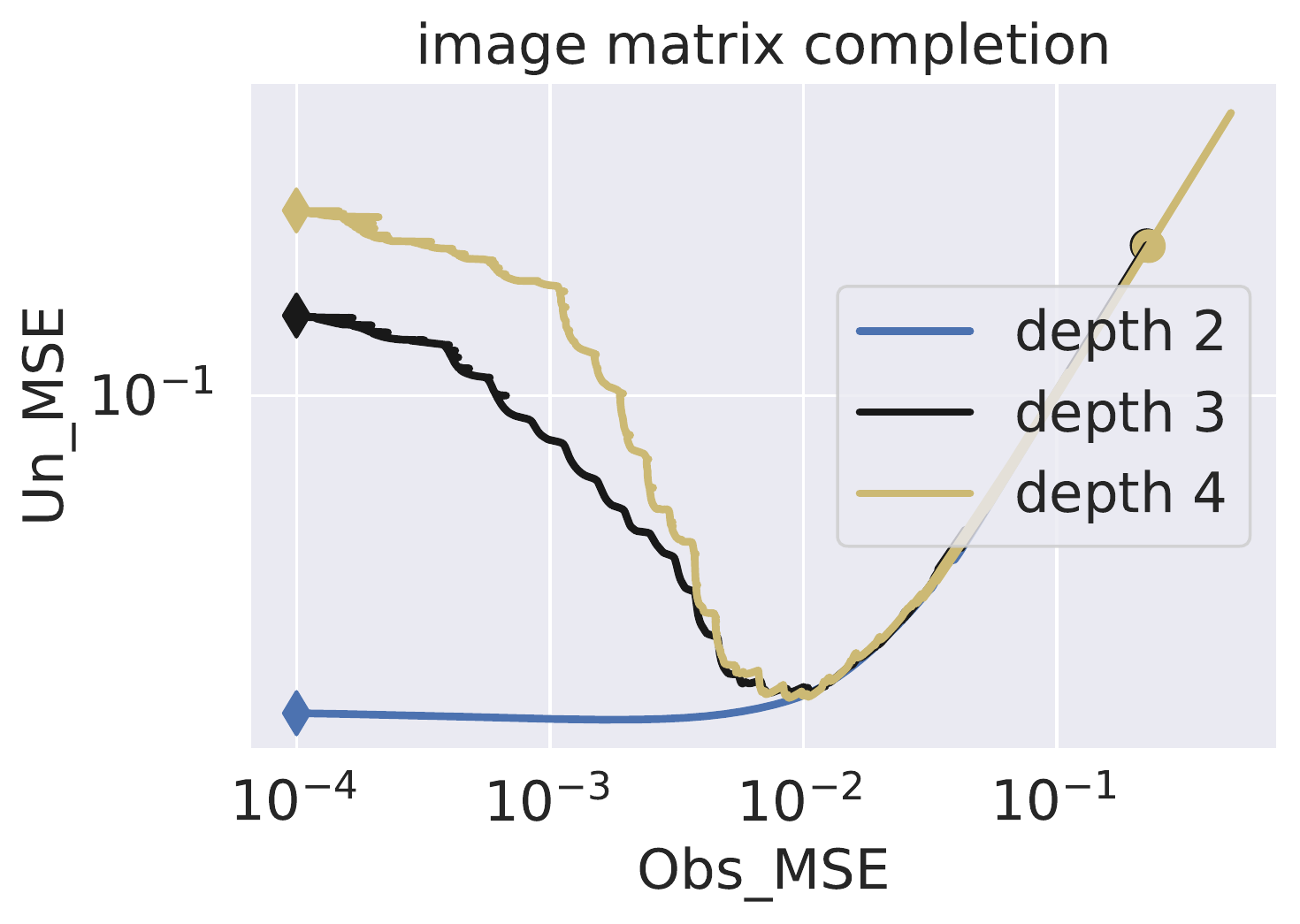}&
\hspace{-0.3cm}\includegraphics[width=0.3\columnwidth]{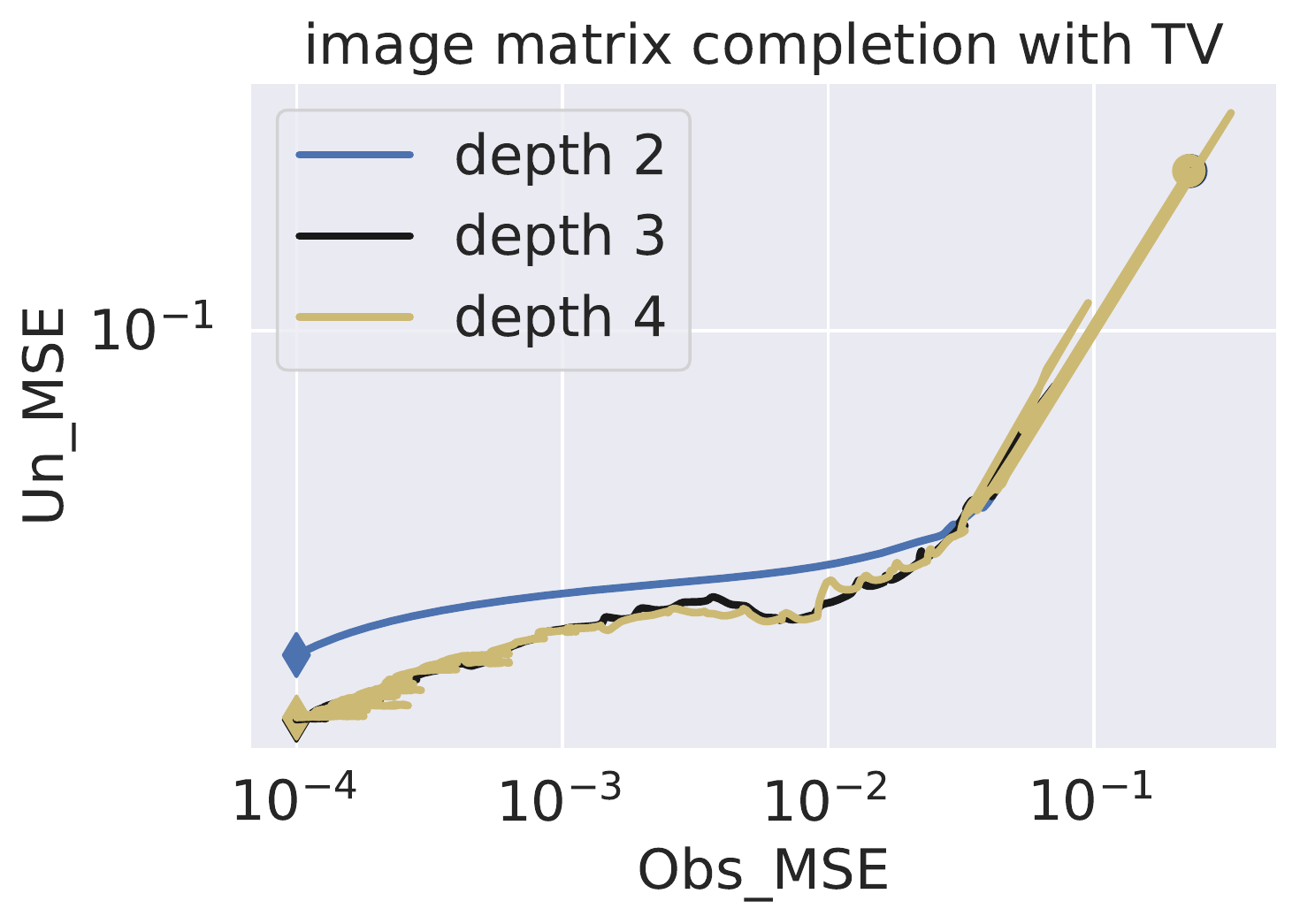}\\[-3pt]
{\footnotesize (a)   Random missing} & {\footnotesize (b) Random missing} & {\footnotesize (c)  Random missing}  \\
\hspace{-0.2cm}\includegraphics[width=0.3\columnwidth]{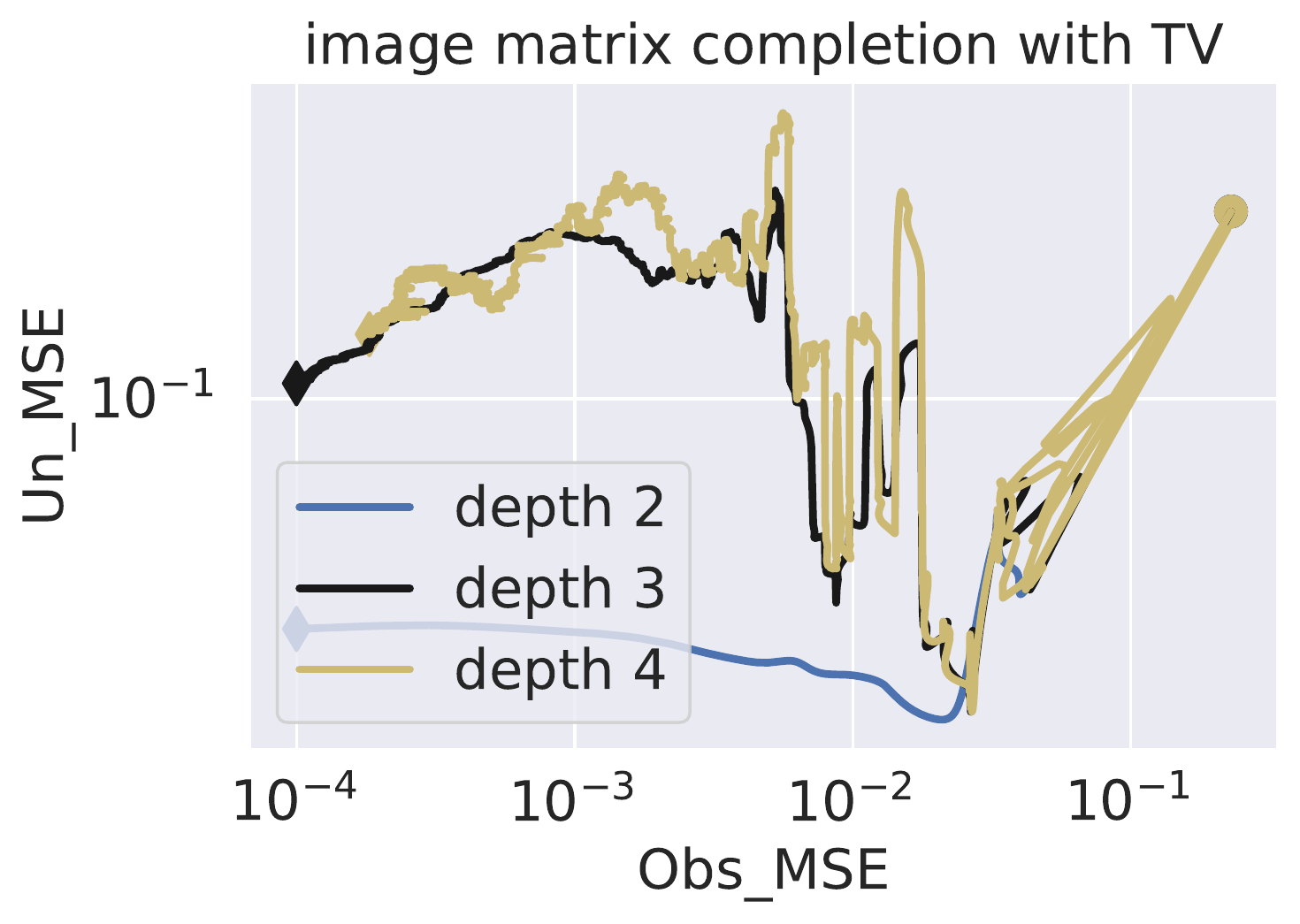}&
\hspace{-0.3cm}\includegraphics[width=0.3\columnwidth]{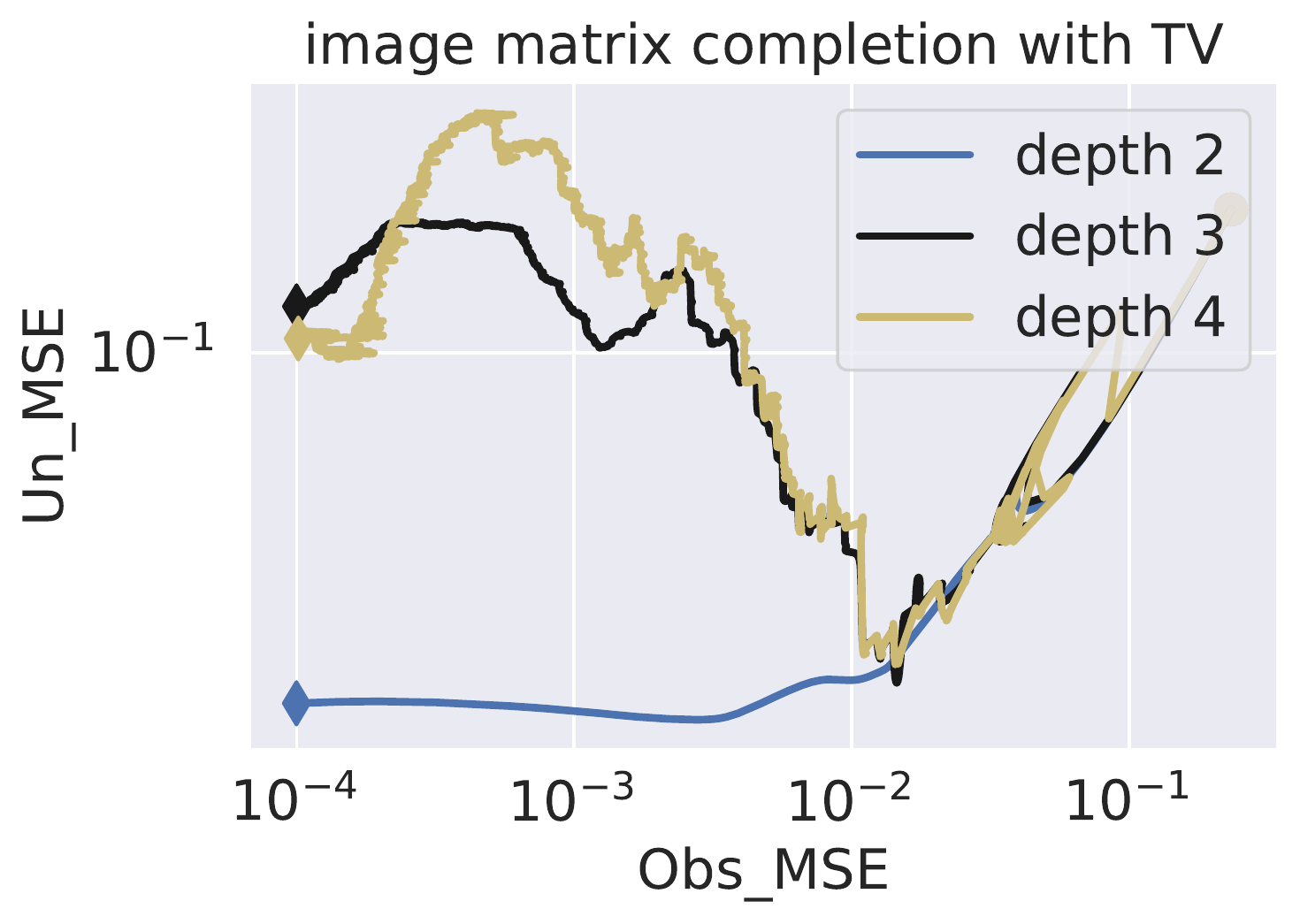}&
\hspace{-0.3cm}\includegraphics[width=0.3\columnwidth]{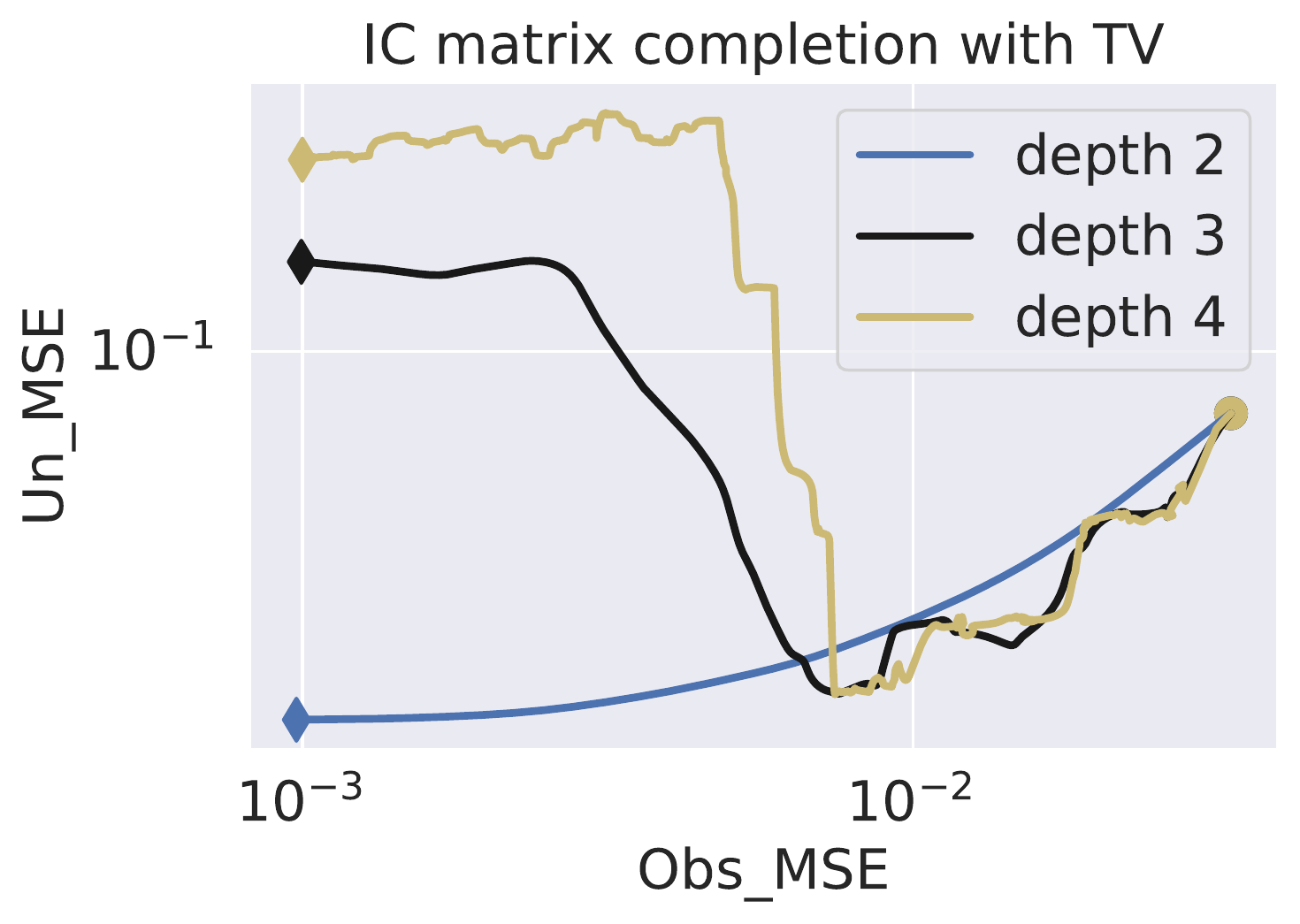}\\[-3pt]
{\footnotesize (d)Patch missing} & {\footnotesize (e) Textural  missing} & {\footnotesize (f)  Random missing}  \\
\end{tabular}
\end{center}\vspace{-0.3cm}
\caption{\footnotesize Trajectories of optimizing DMF with different depths for different matrix completion problems, where we set the mean squared error (MSE) on observed elements to be less than $10^{-3}$ as the stop criterion. Obs\_MSE and Un\_MSE stand for the MSE on observed and unobserved elements, respectively. (a) DMF for completing a rank-5 random matrix of size $100\times 100$ with $80\%$ random missing entries; (b) DMF for completing the gray scale Barbara image, which has the size $240\times 240$ but $80\%$ of its randomly sampled entries are missing; (c) DMF + TV for completing the gray scale Barbara image, which has the size $240\times 240$ but $80\%$ of its randomly sampled entries are missing; (d) DMF + TV for completing the gray scale Barbara image with missing patch (see \Figref{fig:rec_img} (c)), which has the size $240\times 240$; (e) DMF + TV for completing the gray scale Barbara image with missing textural (see \Figref{fig:rec_img} (b)), which has the size $240\times 240$; (f) DMF + TV for completing the IC  (\cite{Boyarski2019SpectralGM}), which has the size $240\times 240$ but $80\%$ of its randomly sampled entries are missing.}\label{fig:WhyDMF}
\end{figure}

\subsubsection{Total variation and Dirichlet energy regularization}
Compared to solving exact low-rank matrix completion problems, directly applying DMF for solving low-rank-related practical problems, including image inpainting, often gives opposite results. As shown in \Figref{fig:WhyDMF} (b), the deep matrix factorization model performs inferior to the shallow models. By investigating the optimization trajectory during the gradient descent procedure, we notice that the major issue that causes deeper models to perform worse than the shallow ones is because the underlying image inpainting problem cannot be precisely formulated as a low-rank matrix completion problem. Indeed, some detailed structures of the matrix correspond to small singular values of the matrix, which the low-rank prior cannot entirely model. To resolve the above issue, we need to enforce other priors to compensate for subtleties in solving image inpainting problems and beyond.

A simple but valuable prior for natural image processing is piece-wise smoothness, which can be induced using the Total Variation (TV) regularization \cite{Rudin1992NonlinearTV,JingQin2021BlindHU,MujiburRahmanChowdhury2020NonblindAB}. 
TV has been integrated with DMF, achieving promising recovery results for matrix completion problems, especially when the rate of missing entries is high \cite{Li2020ARD}. As illustrated in \Figref{fig:WhyDMF} (c), deep DMF models outperform shallow DMF models when TV regularization is enforced. Though TV is a natural prior for natural images, it is suboptimal or inappropriate for many other practical applications, including recommendation systems through matrix completion. Besides the TV regularization, the Dirichlet Energy (DE)-based regularization, defined through the Laplacian matrix \cite{Boyarski2019SpectralGM}, has also been employed for low-rank matrix completion. The DE regularization can describe the similarity between rows and columns of a given matrix. Moreover, we can view the DE regularization as a nonlocal extension of the TV regularization.

\subsection{Do we need a new regularization: Fixed or adaptive}
Given a matrix completion problem, we can integrate several existing explicit regularizers with DMF or other matrix completion models to obtain some successful results. However, the design or selection of explicit regularizers is often task- and missing pattern-dependent. As shown in \Figref{fig:WhyDMF} (d), (e), and (f), although TV regularization performs quite well for random missing patterns and natural images, the performance of using TV regularizers dramatically degrades when they are applied for other missing patterns and data. As such, we need to develop a new regularizer adaptable to different missing patterns and different types of data.

\subsection{Notations}
We denote scalars by lower or upper case letters; vectors and matrices by lower and upper case boldface letters, respectively. For a vector $\vx = (x_1, \cdots, x_d)^\top\in \mathbb{R}^d$, we use $\|\vx\| := {(\sum_{i=1}^d |x_i|^2)^{1/2}}$ to denote its $\ell_2$-norm. We denote the matrix whose entries are all 1s as $\mathbf{1}_{m,n}\in\mathbb{R}^{m\times n}$. For a matrix $\mA$, we use $\mA^\top$,  $\mA^{-1}$, and $\|\mA\|$  to denote its transpose, inverse, and spectral norm, respectively. Given a matrix $\mA$, we denote $\mA_{ij}$ as its $(i,j) $-th entry, and denote its $k$-th column and $k$-th row as $\mA_{:,k}$ and $\mA_{k,:}$, respectively. For a function $f(\vx): \mathbb{R}^d \rightarrow \mathbb{R}$, we denote $\nabla f(\vx)$ as its gradient.

\section{Adaptive Laplacian Regularization}
\label{sec..adaptivereg}
Besides low-rank, self-similarity is another widely used prior for solving inverse problems, and the self-similarity is usually measured by the DE. A particular type of DE is the TV \cite{Rudin1992NonlinearTV}, which describes the piece-wise self-similarity within a given image. To adapt the DE regularization to different missing patterns and different types of data, we raise its degree of freedom by introducing a learnable DE regularization. In particular, we propose the following model
\begin{equation}
\label{eq..AirnetFramework}
    \Min_{\mX,\mW_i}\left\{\gL_{\textrm{all}} := \Bigg(\gL_{\sY}\left(\mY, \gA(\mX)\right) +
	\sum_{i=1}^N \lambda_i \cdot \gR_{\mW_i} \left(\gT_i\left(\mX\right)\right)\Bigg)\right\},
\end{equation}
where the linear operator $\gA(\cdot):\mathbb{R}^{m\times n}\mapsto \mathbb{R}^o$, $\mY\in\sY\subseteq\mathbb{R}^o$ is a measure vector and $\gL_\sY$ is a distance metric on $\sY$. Different from other regularization models for matrix completion, here \emph{$\mX\in\mathbb{R}^{m\times n}$ is parameterized by products of matrix which tends to be low-rank implicitly}, and $\gR_{\mW_i}$ is an adaptive regularizer and $\lambda_i\geq 0$ is a constant. $\gT_i(\cdot)$ is a linear or nonlinear transformation. See \Secref{subsec..AdaReg} for details. A particular case of \EQREF{eq..AirnetFramework} for matrix completion is given in \Secref{subsec..AIRNet} below.

\subsection{Self-similarity and Dirichlet energy}
Given a matrix $\mX\in\sR^{m\times n}$, DE associated with the weighted adjacency matrix $\mA\in\sR^{m\times m}$, along rows of $\mX$, is a classical approach to encode the self-similarity prior of the matrix $\mX$. The adjacency matrix $\mA$ measures the similarity between rows of $\mX$, and $\mA_{ij}$ is bigger if rows $i$ and $j$ of $\mX$ are more similar. The corresponding Laplacian matrix of $\mA$ is $\mL:=\mD-\mA$, where $\mD$ is the degree matrix with $\mD_{ii}=\sum_{j=1}^n\mA_{ij}$ and $\mD_{ij}=0$ if $i\neq j$. Then we can mathematically formulate the DE as follows
$$
\text{tr}(\mX^\top \mL\mX)=\sum_{1\leq i,j\leq m}\mA_{ij}\left\|\mX_{i,:}-\mX_{j,:}\right\|^2.
$$
It is evident that when DE is minimized, rows $\mX_{i,:}$ and $\mX_{j,:}$ become closer with bigger $\mA_{ij}$.

There are two major obstacles to using DE in applications: (i) $\mL$ is unknown for an incomplete matrix, and (ii) DE only encodes the similarity between rows of $\mX$; other similarities such as block similarity cannot be captured. To resolve the above two issues, we parameterize $\mL$ as a function of $\mX$ and learn it during completing $\mX$.

\subsection{Learned Dirichlet energy as an adaptive regularizer}
\label{subsec..AdaReg}
Directly parameterizing $\mL$ without enforcing any particular structure and then optimizing it by minimizing $\text{tr}(\mX^\top \mL\mX)$ might result in all entries of $\mL$ being very negative. To remedy this problem, we first recall a few essential properties of $\mL$ as follows:

\begin{enumerate}
    \item The Laplacian matrix $\mL$ can be generated from the corresponding adjacency matrix $\mA$, that is $\mL=\mD-\mA$, where $\mD$ is the degree matrix. 
    \item $\mA$ is a symmetric matrix and $\mA_{ij}\geq 0$ for all $i,j$.
\end{enumerate}

One of the most straightforward approaches is to model $\mA$ as the covariance matrix, e.g., $\mA:=\mX^\top\mX$, of the rows of $\mX$. However, this model may not capture sufficient similarities among rows of 
$\mA$. To this end, the self-attention mechanism \cite{vaswani2017attention} has been proposed to model similarities among different rows. However, the learning self-attention mechanism usually requires a large amount of training data, which is not the case for the matrix completion problem. Therefore, we cannot leverage the off-the-shelf self-attention mechanism for matrix completion.

We propose a new parameterization of the adjacency matrix that is similar to but simpler than the self-attention mechanism. Our new parametrization can capture row similarities of a given matrix beyond the covariance. In particular, we parameterize $\mA$ as follows
\begin{equation}\label{eq:parameterization}
	\left\{
	\begin{array}{l}
	\mA=\frac{\exp(\mW+\mW^\top)}{\mathbf{1}_m^\top \exp(\mW)\mathbf{1}_m}\\
		\mL=\left(\mA\cdot \mathbf{1}_{m\times m}\right)\odot \mathbf{I}_m-\mA
	\end{array}
	\right. ,
\end{equation}
where $\mathbf{1}_m$ is an $m$-dimensional vector whose entries are all 1s, $\mathbf{I}_m$ is the $m\times m$ identity matrix, $\mathbf{1}_{m\times m}$ is the $m\times m$ matrix whose entries are all 1s and $\odot$ is Hadamard product. $\mW\in\sR^{m\times m}$ has the same size as $\mA$ and $\exp(\cdot)$ denotes the element-wise exponential. Note that the above parameterization of $\mA$ guarantees $\mA$ to be symmetric, and all its entries are positive. To search for the optimal $\mA$, we need to update $\mW$ start from a certain initialization. It is clear that the minimum value of DE is 0, i.e., $\text{tr}(\mX^\top \mL\mX)=0$, which is obtained when $\mA=\mL={\bf 0}$ and this is known as the trivial solution. Our parameterization \EQREF{eq:parameterization} directly avoids the above trivial solution since $\mA_{ij}>0$. In Theorem \ref{thm..ConvReg}, we will show that $\mL$ will converge to a non-trivial minimum under the gradient descent dynamics using the parameterization in \EQREF{eq:parameterization}.

To capture more types of self-similarity, we further replace $\mX$ with 
$\gT_i\left(\mX\right)$, where $\gT_i(\cdot)$ is a linear or nonlinear transformation that will be specified in particular applications. And we can generalize the above adaptive regularizer as follows

\begin{equation}\label{eq:adp-reg-transform}
\gR_{\mW_i} \left(\gT_i\left(\mX\right)\right)=\mathrm{tr}\left(\left[\gT_i\left(\mX\right)\right]^{\top} \mL_i\left(\mW_i\right)\gT_i\left(\mX\right) \right), i=1,2,\ldots,N,
\end{equation}
where $\mL_i\in\sR^{m_i\times m_i}$ is parameterized by $\mW_i\in\sR^{m_i\times m_i}$ following \EQREF{eq:parameterization}. 

$\gT_i:\sR^{m\times n}\mapsto \sR^{m_i\times n_i}$ transforms $\mX$ into another domain, making the adaptive regularization able to capture different similarities in data.
The common choice of $\gT_i(\cdot)$ can be $\gT_i\left(\mX\right)=\mX$ and $\gT_i\left(\mX\right)=\mX^\top$, which capture the row and column correlation, respectively.
Another particularly interesting case is 
$\gT_i(\mX)=\left[\textbf{vec}\left(\text{block}(\mX)\right)_1,\ldots,\textbf{vec}\left(\text{block}(\mX)\right)_{m_i}\right]^\top$, where $\textbf{vec}\left(\text{block}(\mX)\right)_j\in\sR^{n_i}$ is the vectorization of $j$-th block of $\mX$ (we divide $\mX$ into $m_i$ blocks), then the similarity among blocks can be modeled. 
Apart from these handcraft transformations to model specific structured correlations of the image, we can also parameterize $\gT_i$ with a neural network to learn more broad classes of correlations from data.  
However, we noticed that learning $\gT_i$ using a neural network is a very challenging task, which we leave as future work. This work focuses on handcrafted transformations $\gT_i$ that capture similarities among rows and columns of the matrix $\mX$.
Another interesting result is that $\gR_{\mW_i}$ vanishes at the end of the training, which avoids over-fitting without early stopping. We show that $\gR_{\mW_i}$ vanishes at the end of training both theoretically (see Theorem \ref{thm..ConvReg}) and numerically (see \Secref{sec..ExpAna}).


\subsection{Adaptive regularizer for matrix completion}
\label{subsec..AIRNet}
This subsection considers applying our proposed adaptive regularizer for matrix completion. In particular, we consider two transformations:
$\gT_1\left(\mX\right)=\mX$ and $\gT_2\left(\mX\right)=\mX^{\top}$ 
to capture 
row and column similarities of $\mX$, respectively.
We denote the corresponding subscript of row and column as $r$ and $c$, respectively.
And the DMF model with the regularization in \EQREF{eq..AirnetFramework} becomes
\begin{equation}\small
\label{eq..DMF_AIR_Completion}
\begin{aligned}
\Min{\boldsymbol{\mW^{[l]},\mW_r,\mW_c}}&\left\{\gL:=\Bigg(
    \frac{1}{2}\left\|\mY-\gA\left(\prod_{l=0}^{L-1}\mW^{[l]}\right) \right\|_F^2\right.\\
    &\left.+\lambda_r\cdot \gR_{\mW_r}\left(\prod_{l=0}^{L-1}\mW^{[l]}\right)+\lambda_c\cdot \gR_{\mW_c}\left(\left(\prod_{l=0}^{L-1}\mW^{[l]}\right)^{\top}\right)\Bigg)\right\},
\end{aligned}
\end{equation}
for $l=0,1,\cdots, L-1.$ We name (\ref{eq..DMF_AIR_Completion}) as \textbf{Adaptive and Implicit Regularization (AIR)}.
The parameters in AIR are updated by using gradient descent or its variants. 
We stop the iteration when $\left|\gR_{\mW_r (T+1)}-\gR_{\mW_r (T)}\right|<\delta$ and $\left|\gR_{\mW_c (T+1)}-\gR_{\mW_c (T)}\right|<\delta$. The recovered matrix is $\hat{\mX}(T)=\mW^{[L-1]}(T)\cdots \mW^{[0]}(T)$, which tends to be low-rank implicitly \cite{Arora2019ImplicitRI}. With the implicit low-rank property, we do not need to estimate the shared dimension of $\mW^{[l]}$ in advance (more details can be found in Appendix \ref{sec..deepwidth}). The framework of algorithm is shown in \Figref{fig:flowchart}.

Some works that combine implicit and explicit regularization can be regarded as special cases of \EQREF{eq..AirnetFramework}. For instance, both 
TV \cite{Rudin1992NonlinearTV} and DE \cite{Mullen2008SpectralCP} can be 
considered as fixed $\mL$ in \EQREF{eq..AirnetFramework}. Therefore, our proposed framework in \EQREF{eq..AirnetFramework} contains DMF+TV \cite{Li2020ARD} and DMF+DE \cite{Boyarski2019SpectralGM}.

\begin{figure}
    \includegraphics[width=\linewidth]{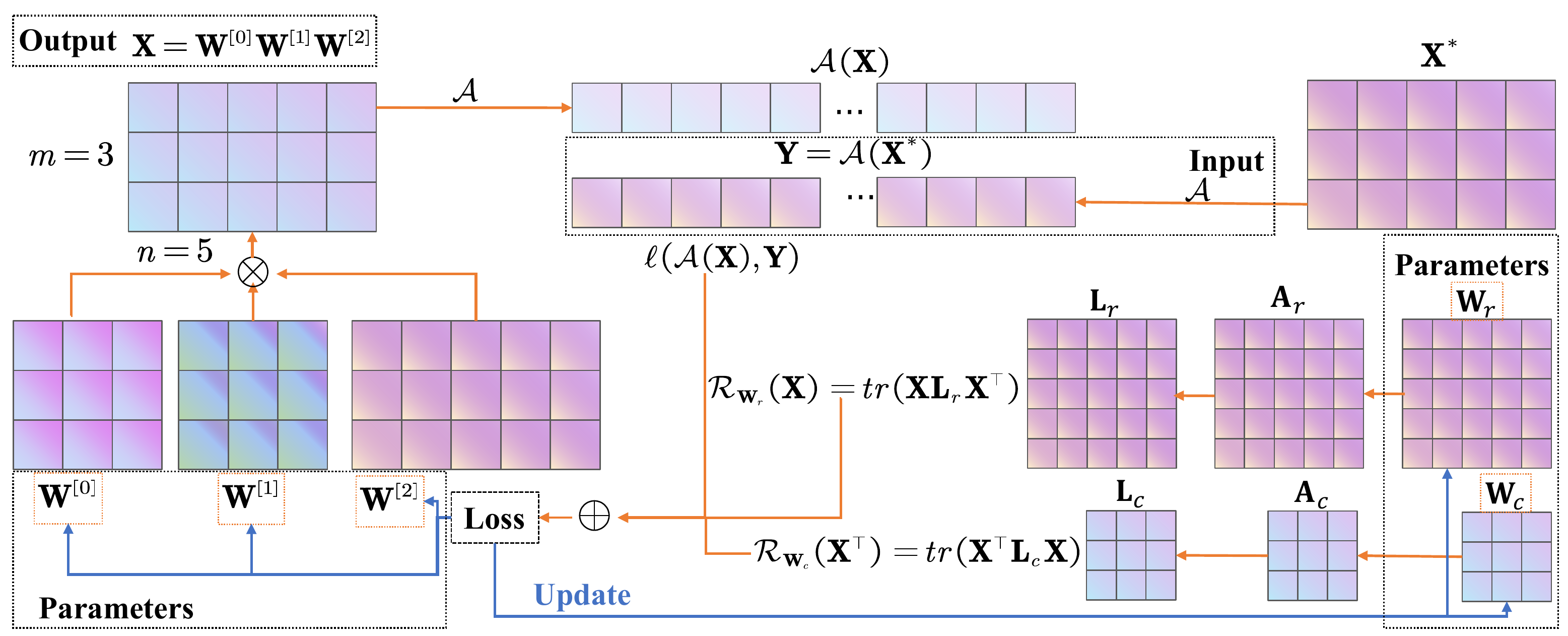}
    \caption{We give an example of AIR when completing a $3\times 5$ matrix. The $\gA$ and $\mY$ are inputs of our algorithm. In this case, we factorized $\mX$ into three matrices. The $\mW_r$ and $\mW_c$ are parameters of AIR. In each iteration, we calculate the loss function based on these parameters and update them with an optimization algorithm such as gradient descent.}
    \label{fig:flowchart}
\end{figure}

\section{Theoretical Analysis of AIR}
\label{sec..TheoreticalAna}
In this section, we will analyze the theoretical properties of
solving problem \EQREF{eq..DMF_AIR_Completion} using gradient descent.
In particular, we will show that (a) AIR enhances the implicit low-rank of DMF (see Theorem \ref{thm..ImplicitReg} below); and (b) the adaptive regularization will converge to a minimum under gradient descent, and the resulting model captures the intrinsic structure of data flexibly. Although we focus on matrix completion problems, the following theoretical analyses apply to general inverse problems.
For the ease of notation, as $\gA$ and $\mY$ are fixed during optimization, we denote $\gL_{\sY}\left(\mY,\gA(\mX)\right)$ as $\gL_{\sY}\left(\mX\right)$ below.

\subsection{Adaptive regularizer has implicit low-rank regularization}
We show that the proposed adaptive regularizer introduces implicit low-rank regularization. Below we denote the Laplacian matrices parameterized by $\mW_r$ and $\mW_c$, following \EQREF{eq:parameterization}, as $\mL_r$ and $\mL_c$, respectively. 

\begin{thm}\label{thm..ImplicitReg}
Consider the following dynamics with initial data satisfying the balanced initialization in Assumption \ref{assump..balance} (see Appendix~\ref{app..IntroDMF} for details)
    \begin{align*}
        \dot{\mW}^{[l]}(t)
        = -\frac{\partial}{\partial \mW^{[l]}} \gL(\mX(t)), \quad t \geq 0, \quad l=0, \ldots, L-1,\label{eq..DynamicsReg}
    \end{align*}
    where $ \gL\left(\mX\right) := \gL_{\sY}(\mX)+\lambda_r\cdot \gR_{\mW_r}\left(\mX\right)+\lambda_c\cdot \gR_{\mW_c}\left(\mX\right)$. Then for $k=1,2,\ldots$, we have
    \begin{equation}
    \label{eq..DynamicArora}
    \begin{aligned}
        \dot{\sigma}_k(t)=&-L \left(
        \sigma_k^{2}(t)\right)^{1-\frac{1}{L}} 
         \left\langle\nabla_{\mW} \gL_{\sY}(\mX(t)), \mU_{k,:}(t)\left[\mV_{k,:}(t)\right]^\top\right\rangle
        -2L \left(\sigma_k^2(t)\right)^{\frac{3}{2}-\frac{1}{L}} \gamma_k(t),
    \end{aligned}
    \end{equation}
    where 
    $\mU(t) \mS(t) \mV(t)^\top$ is the SVD of the matrix $\mX\left(t\right)$, $\sigma_k(t)$ is the $k$-th singular value in descending order of $\mX(t)$, $\mX=\prod_{l=0}^{L-1}\mW^{[l]}=\sum\limits_s\sigma_s \mU_{s,:} \mV_{s,:}^{\top},\gamma_k(t)=\mU_{k,:}^{\top}\mL_r\mU_{k,:}+\mV_{k,:}^{\top}\mL_c\mV_{k,:}\geq 0$.
\end{thm}
\begin{proof}
	By direct calculation, we have
	$$
	\begin{aligned}
	\nabla_{\mX}\left(\lambda_r\gR_{\mW_r}+\lambda_c\gR_{\mW_c}\right)
	&=\frac{\partial \mathrm{tr}\left(\lambda_r\mX^{\top}\mL_r\mX+\lambda_c\mX\mL_c\mX^{\top}\right)}{\partial \mX}
	=2\lambda_r\mL_r\mX+2\lambda_c\mX\mL_c\\
	&=2\lambda_r\mL_r\sum\limits_s\sigma_s \mU_{s,:} \mV_{s,:}^{\top}+2\lambda_c\sum\limits_s\sigma_s \mU_{s,:} \mV_{s,:}^{\top} \mL_c.
	\end{aligned}
	$$
	Note that
	$$
	\langle\mV_{s,:},\mV_{s',:}\rangle=\langle\mU_{s,:},\mU_{s',:}\rangle=\delta_{ss'}=
	\left\{
	\begin{array}{cc}
	1, & s=s'\\
	0, & s\neq s'\\
	\end{array}\right. ,$$
	Therefore
	$$
	\begin{aligned}
	\mU_{k,:}^{\top}\left(\nabla_{\mX}\left(\lambda_r \gR_{\mW_r}+\lambda_c \gR_{\mW_c}\right)    \right)\mV_{k,:} 
	&= 2\sigma_k(\lambda_r\mU_{k,:}^{\top}\mL_r \mU_{k,:}+\lambda_c\mV_{k,:}^{\top}\mL_c\mV_{k,:})
	=2\sigma_k\gamma_k(t),
	\end{aligned}
	$$
	 where the term $\gamma_k(t)=2\sigma_k(\lambda_r\mU_{k,:}^{\top}\mL_r \mU_{k,:}+\lambda_c\mV_{k,:}^{\top}\mL_c\mV_{k,:})\geq 0$.
	 Furthermore, according to \EQREF{eq..arora} in the appendix, we have 
	 $$
	 \mU_{k,:}^{\top}\nabla_{\mX}\gL\mV_{k,:}=-L \left(
        \sigma_k^{2}(t)\right)^{1-\frac{1}{L}} 
         \left\langle\nabla_{\mX}\gL(\mX(t)), \mU_{k,:}(t)\left[\mV_{k,:}(t)\right]^\top\right\rangle.
         $$
    Combined the above results with $\gL(\mX)=\gL_{\sY}(\mX)+\lambda_r\gR_{\mW_r}+\lambda_c\gR_{\mW_c}$, we have 
        $$\dot{\sigma}_k(t)=-L \left(
        \sigma_k^{2}(t)\right)^{1-\frac{1}{L}} 
         \left\langle\nabla_{\mX} \gL_{\sY}(\mX(t)), \mU_{k,:}(t)\left[\mV_{k,:}(t)\right]^\top\right\rangle
        -2L \left(\sigma_k^2(t)\right)^{\frac{3}{2}-\frac{1}{L}} \gamma_k(t),$$
which completes the proof.
\end{proof}

Compared with the result of the vanilla DMF whose order of $\sigma_k(t)$ is $2-\frac{2}{L}$, Theorem \ref{thm..ImplicitReg} demonstrates that AIR's $\sigma_k(t)$ has a higher dynamics order 
$3-\frac{2}{L}$. Note that the adaptive regularizer keeps $\gamma_k(t)\geq 0$. This way, a faster convergence rate gap appears between different singular values $\sigma_r$ than the vanilla DMF. Therefore, AIR enhances the implicit low-rank regularization over DMF.

\subsection{The auto vanishing property of the adaptive regularizer}
We suppose the matrix $\mX$ is fixed and then study the convergence of $\gR_{\gT_i(\mW_i)}$ based on the dynamics of $\mL_i$. Theorem \ref{thm..ConvReg} shows that $\gR_{\gT_i(\mW_i)}$ vanishes at the end and thus prevents AIR from suffering over-fitting. In the process of proof of the following, we can replace the $\gT_i,\mW_i$ with $\gT,\mW$ in the following proof.

\begin{thm}\label{thm..ConvReg}
    Consider the following gradient flow model,
    $$
    \dot{\mW}(t)=-\nabla_{\mW(t)}\gR_{\mW(t)} \left(\gT\left(\mX(t)\right)\right)=-\nabla_{\mW(t)}\mathrm{tr}\left(\gT\left(\mX(t)\right)^{\top} \mL\left(\mW(t)\right) \gT\left(\mX(t)\right)\right),
    $$
    where $\|\gT\left(\mX\right)_{k,:}\|_F^2=1$ and $\gT\left(\mX\right)_{kl}>0$. If we initialize $\mW\left(0\right)=\varepsilon \mathbf{1}_{m_i\times m_i},\forall \varepsilon\in\mathbb{R}$, then $\mW\left(t\right)$ will keep symmetric during the optimization procedure. Furthermore, we 
    have the following element-wise convergence 
    $$
    \left|\mL_{kl}(t)-\mL_{kl}^*\right|\leq 
    \left\{\begin{array}{cc}
        2 \exp(-D t)/\gamma, & (k,l)\in\sS_1\\
        \exp\left(-D t\right), & (k,l)\in\sS_2 \\
        2\left(m-1\right) \exp(-D t)/\gamma,  & k=l
    \end{array}\right. ,
    $$ 
    where 
    $$\sS_1 = \left\{(k,l)\mid k\neq l,\gT\left(\mX\right)_{k,:}\neq \gT\left(\mX\right)_{l,:}\right\},\quad 
    \sS_2=\left\{(k,l)\mid k\neq l,\gT\left(\mX\right)_{k,:}= \gT\left(\mX\right)_{l,:}\right\},$$
    and
    $$
    \mL_{kl}^*=\left\{\begin{array}{cc}
        0, &  (k,l)\in\sS_1\\
        \gamma, &  (k,l)\in\sS_2\\
        -\sum_{l'=1,l'\neq l}^{m} \mL_{kl'}^*,& k=l 
    \end{array}\right. .
    $$
    $\mL_{kl}(t)$ is the $(k,l)$-th 
    element of $\mL(t)$, $\mathbf{1}_{m\times m}$ is a matrix whose entries are all 1s. 
    $\gamma<1,D\geq 0$ are constant defined in Appendix \ref{app..Theorem2} and $D$ equals zero if and only if $\gT\left(\mX\right)=\textbf{1}_{m_i\times n_i}$. 
\end{thm}
\begin{proof}
See Appendix \ref{app..Theorem2}.
\end{proof}

The assumptions in Theorem \ref{thm..ConvReg} can be satisfied by applying appropriate normalization and linear transformation on $\mX$. Theorem \ref{thm..ConvReg}
gives the limit
and convergence rate of $\mL(t)$. 
In particular, the limit satisfies 
$\mL^*_{kl}=0$ unless $\gT\left(\mX\right)_{k,:}= \gT\left(\mX\right)_{l,:}$ or $k=l$.
That is, 
$\mathbf{L}^*_{i,j}$ reflects if the $i$-th and $j$-th rows of the matrix $\mathcal{T}\left(\mathbf{X}\right)$ are the same or not. $\mathbf{L}^*_{i,j}\neq 0$ if and only if the $i$-th and $j$-th rows of $\mathcal{T}\left(\mathbf{X}\right)$ are identical.
$\mL_{kl}(t)$ converges faster for those $(k,l)\in\sS_2$ than that in $\sS_1$. 
In other words, the adaptive regularizer first captures structural similarities in the matrix $\mX$. 
The different convergence rates for different $(k,l)$ indicate AIR generates a multi-scale similarity, which will be discussed in \Secref{subsec..spatial}. 
Moreover, 
$\gR_{\mW}(t)$ converges to 0 as shown in the following Corollary. The adaptive regularization will vanish at the end of training and not cause over-fitting.

\begin{cor}\label{cor..ConvReg}
    In the setting of Theorem \ref{thm..ConvReg}, we further have 
    $$0\leq \gR_{\mW}(t)\leq 2\left(m-1\right)m \exp(-Dt)/\gamma,i=1,2,\ldots,N.$$
\end{cor}
\begin{proof}
Direct computation gives
$$
\begin{aligned}
\gR_{\mW}(t)& =\sum_{k,l}\mL_{kl}(t)\left\|\gT(\mX)_{k,:}-\gT(\mX)_{l,:}\right\|_2^2\\
& \leq 2 \sum_{k\neq l} \mL_{kl}(t)\leq 2m\left(m-1\right)\exp(-Dt)/\gamma,
\end{aligned}
$$
which concludes the proof.
\end{proof}

We have shown that AIR can enhance implicit low-rank regularization and avoid over-fitting. In the next section, we will validate these theoretical results and the effectiveness of AIR numerically.

\section{Experimental Results}
\label{sec..ExpAna}
In this section, we numerically validate the following adaptive properties of AIR: (a) Laplacian matrices $\mL_r$ and $\mL_c$ capture the structural similarity in data from large scale to small scale; 
(b) the capability of capturing structural similarities at all scales is crucial to the success of matrix completion, which confirms the importance of our proposed adaptive regularizers;
(c) AIR is adaptive to different data, avoids overfitting, and achieves remarkable performance; (d) AIR behaves like the momentum that has been widely used in optimization \cite{polyak1964some,nesterov1983method,wang2020scheduled,wang2020stochastic,sun2021training} and neural network architecture design \cite{MomentumRNN,HBNODE2021}.

\textbf{Data types and sampling patterns.}
We consider completion of three types of matrices: 
gray-scale image, user-movie rating matrix \cite{Boyarski2019SpectralGM}, and drug-target interaction (DTI) data \cite{Boyarski2019SpectralGM}. In particular, we consider three benchmark images of size  $240\times 240$ \cite{Monti2017GeometricMC}, including Baboon, Barbara, and Cameraman.
We consider the Syn-Netflix dataset for the user-movie rating matrix, which has a size of $150\times 200$. The DTI data describes the interaction of Ion channels (IC) and G protein-coupled receptor (GPCR), which have the shape of $210\times 204$ and $223\times 95$, respectively \cite{Boyarski2019SpectralGM,Mongia2020DrugtargetIP}. 
Moreover, we study matrix completion with three different missing patterns: random missing, patch missing, and textural missing, see \Figref{fig:rec_img} for an illustration.
The random missing rate varies in different experiments, and the default missing rate is $30\%$.

\textbf{Parameter settings.}
We set $\lambda_r=\lambda_c={(\mY_{\max}-\mY_{\min})}/{(m n)}$ to ensure 
both fidelity and 
regularization have similar order of magnitude, 
where $\mY_{\max}$ and $\mY_{\min}$ are the maximum and the minimum entry of $\mY$, respectively. 
$\delta$ is a threshold, which has the default value of ${mn}/{1000}$.
All the parameters in AIR are initialized with Gaussian initialization of zero mean and variance $10^{-5}$. We use Adam \cite{Kingma2015AdamAM} to train the AIR.

\subsection{Multi-scale similarity captured by adaptive regularizer}\label{subsec..spatial}
In this subsection, we will validate Theorems~\ref{thm..ImplicitReg} and \ref{thm..ConvReg}. 
We focus on using AIR for inpainting the corrupted Baboon image and completing the Syn-Netflix matrix. \Figref{fig:baboon} and \Figref{fig:rela_recom} plot the heatmaps of Laplacian matrices $\mL_r(t)$ and $\mL_c(t)$ for Baboon and Syn-Netflix experiments, respectively. These heatmaps show what AIR learns during training.

\begin{figure}[htp]
    \hspace{1cm}\includegraphics[width=0.9\linewidth]{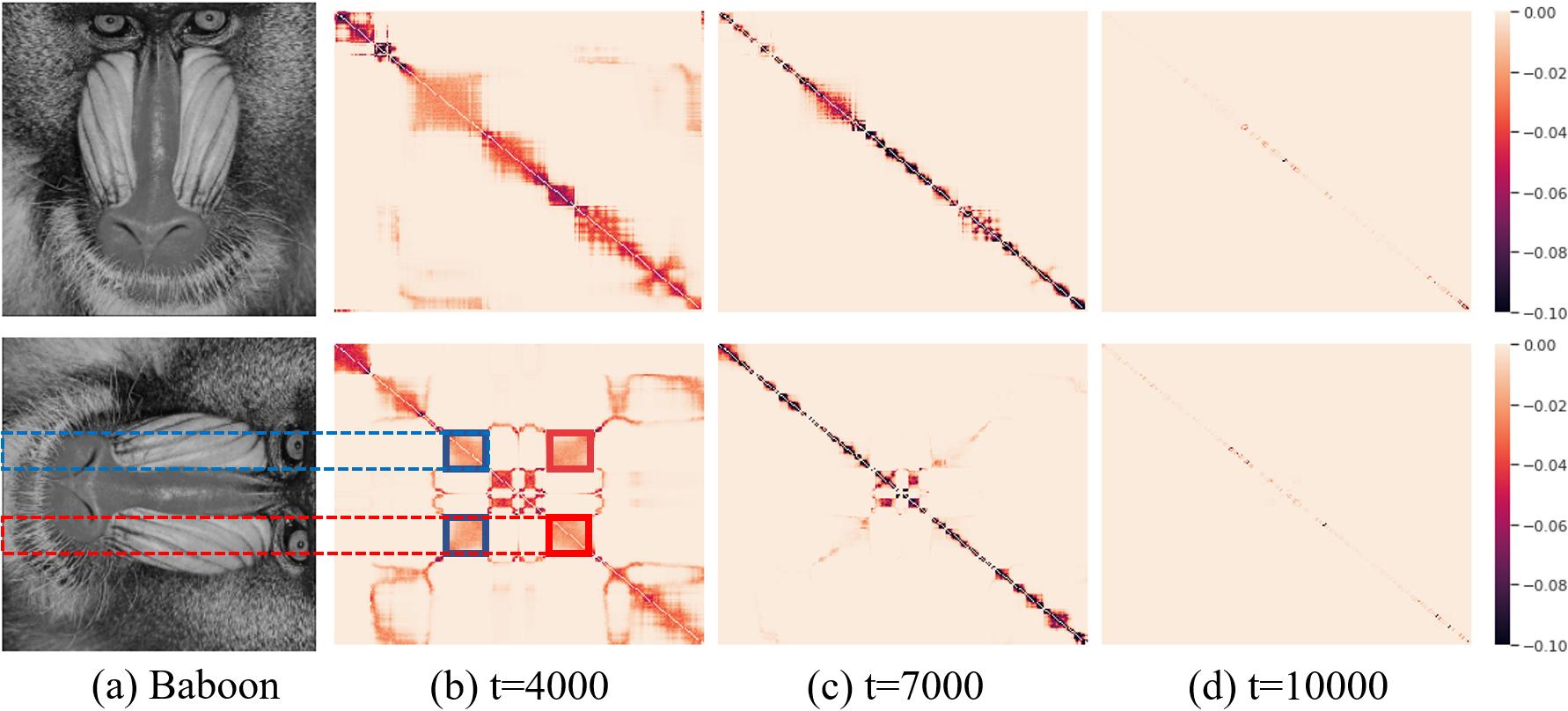}\vspace{-0.3cm}
    \caption{\footnotesize (a): first and second rows depict the Baboon image and its rotation. (b)-(d): first/second row shows the heatmap of $\mL_r$/$\mL_c$ at different $t$s. A darker color indicates a stronger similarity captured by the adaptive regularizer. The $(i,j)$-th element in the heatmap of $\mL_r(t)$ has a darker color than the $(i,j')$-th element indicates that the $i$-th row is more related to $j$-th row compared with $j'$-th row. 
    }
    \label{fig:baboon}
\end{figure}

\begin{figure}
\begin{center}
\begin{tabular}{cccc}
\hspace{-0.2cm}\includegraphics[width=0.25\columnwidth]{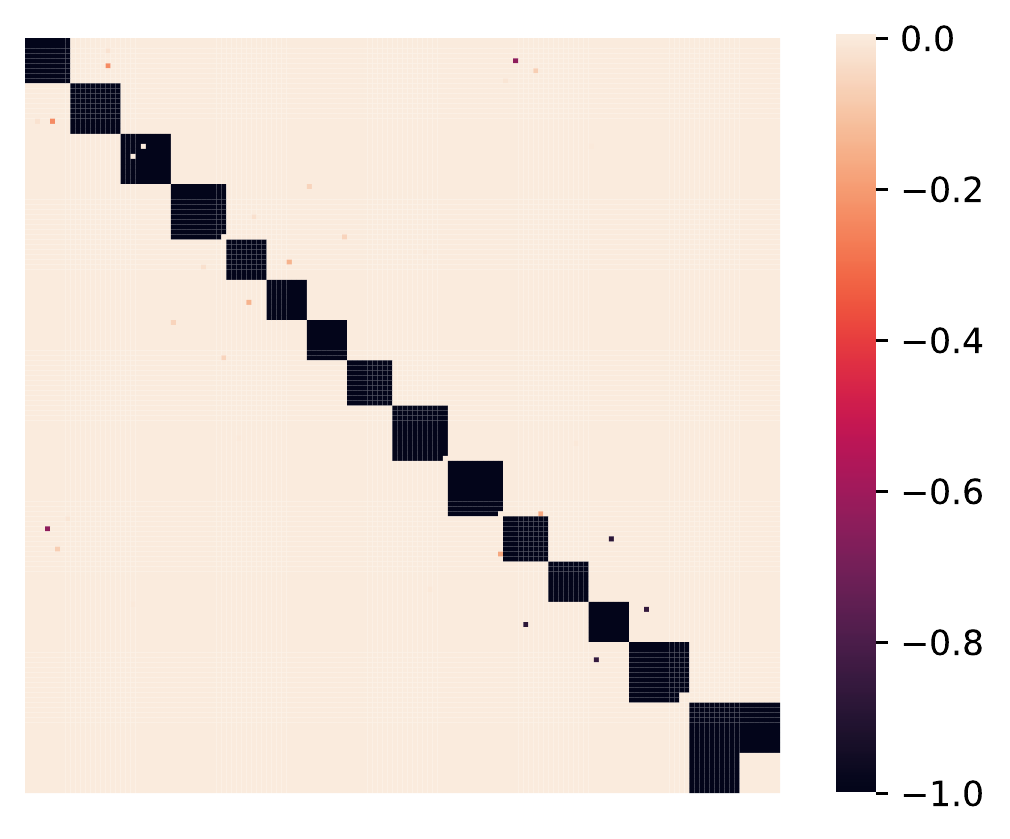}&
\hspace{-0.2cm}\includegraphics[width=0.2\columnwidth]{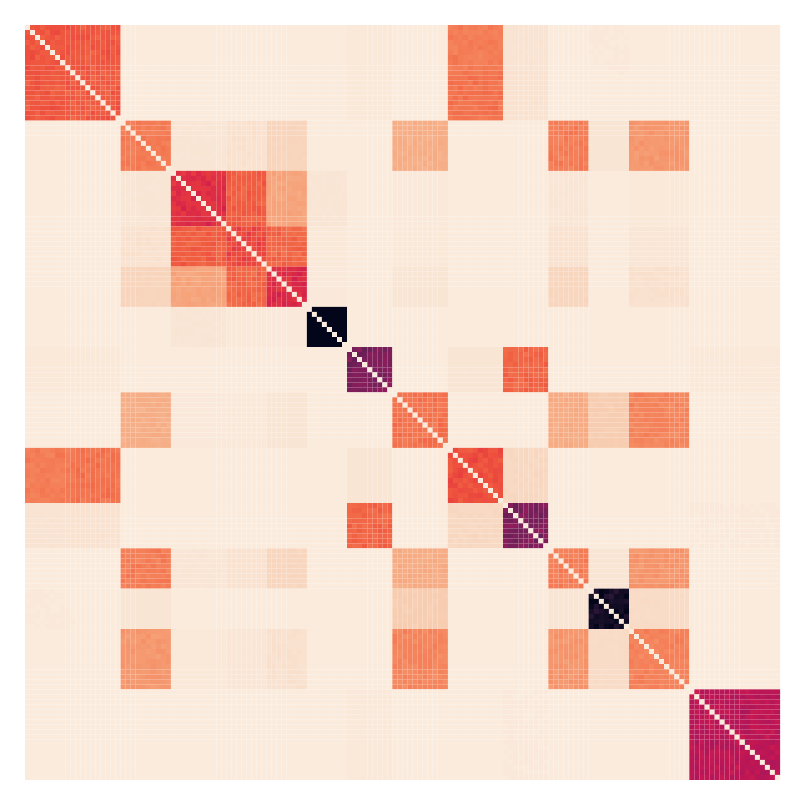}&
\hspace{-0.3cm}\includegraphics[width=0.2\columnwidth]{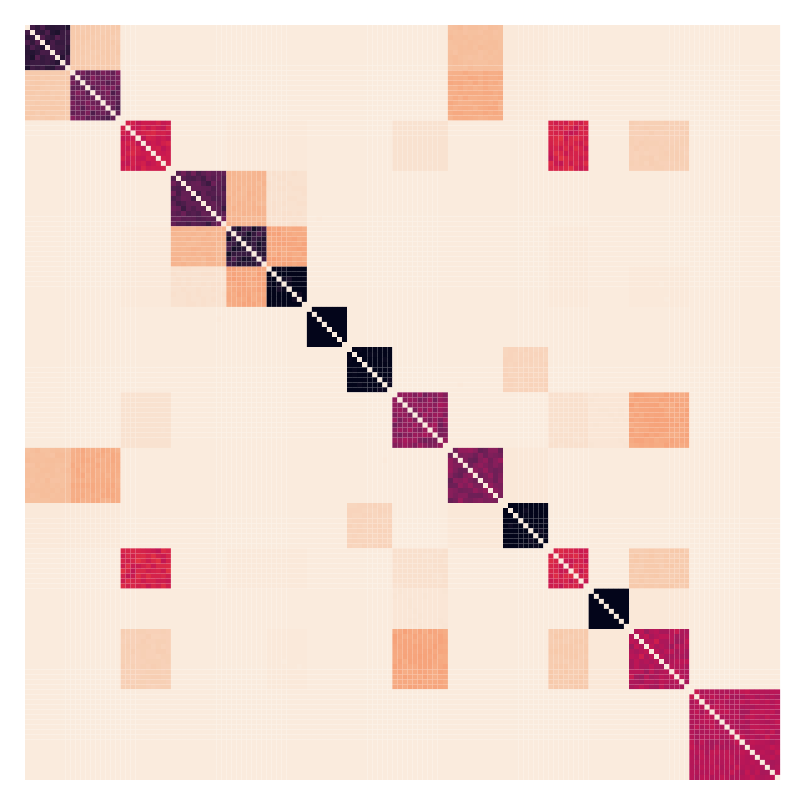}&
\hspace{-0.3cm}\includegraphics[width=0.2\columnwidth]{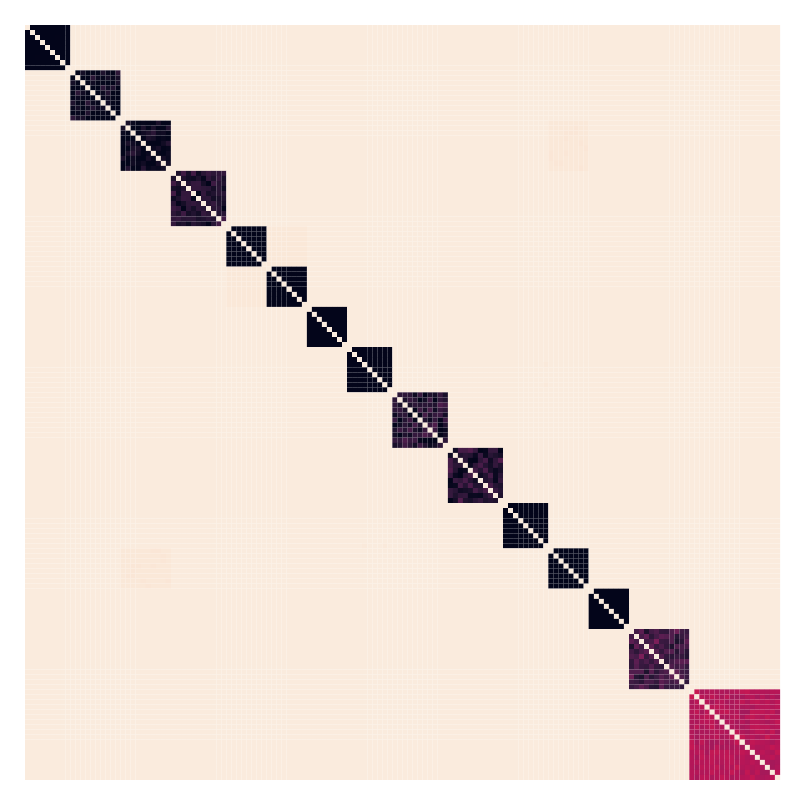}\\[-2pt]
\hspace{-0.2cm} {\footnotesize (a) Users relationship} &\hspace{-0.2cm} {\footnotesize (b) $\mL_r(t=5000)$} &\hspace{-0.2cm} {\footnotesize (c) $\mL_r(t=7000)$} &\hspace{-0.2cm} {\footnotesize (d) $\mL_r(t=10000)$} \\
\hspace{-0.2cm}\includegraphics[width=0.25\columnwidth]{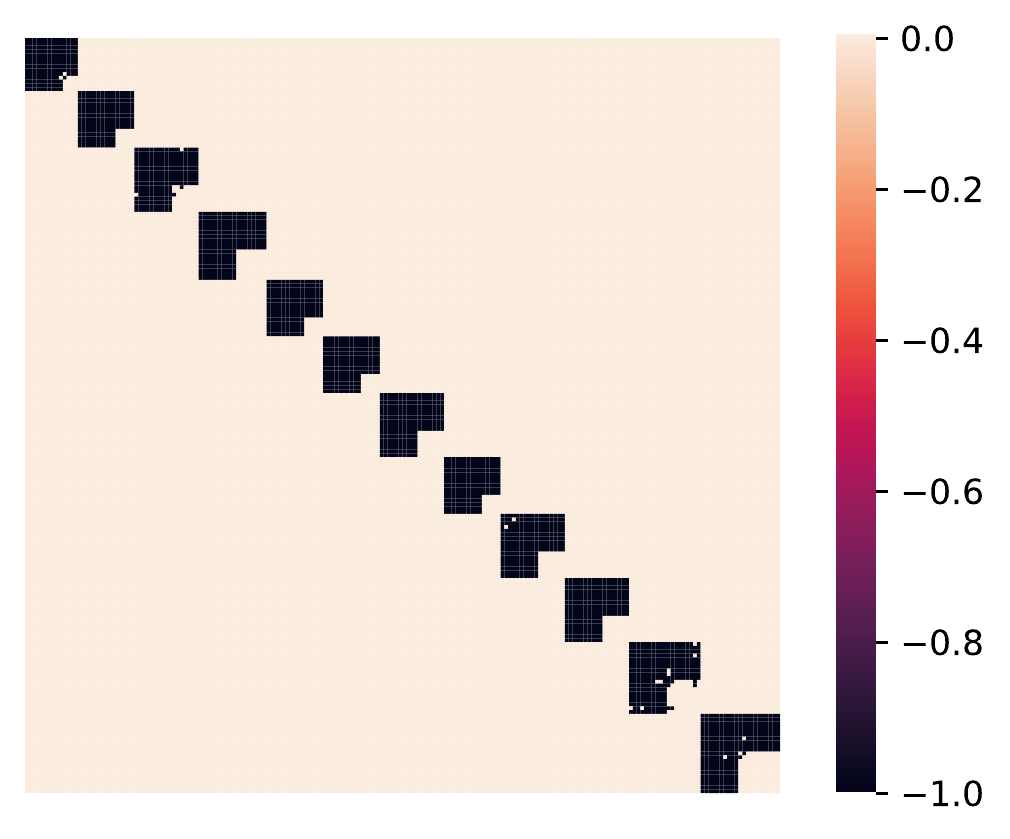}&
\hspace{-0.2cm}\includegraphics[width=0.2\columnwidth]{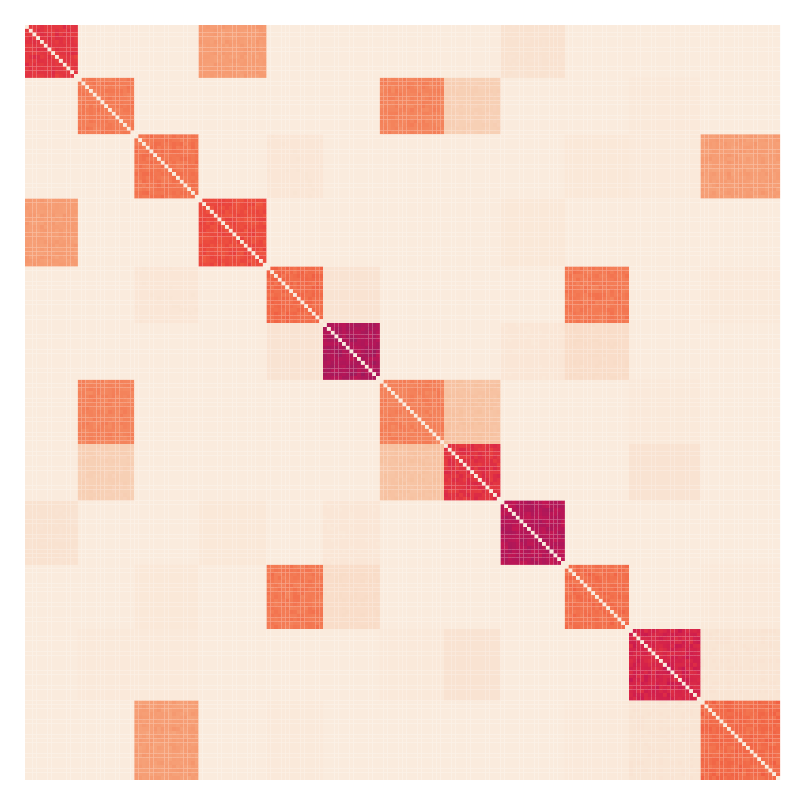}&
\hspace{-0.3cm}\includegraphics[width=0.2\columnwidth]{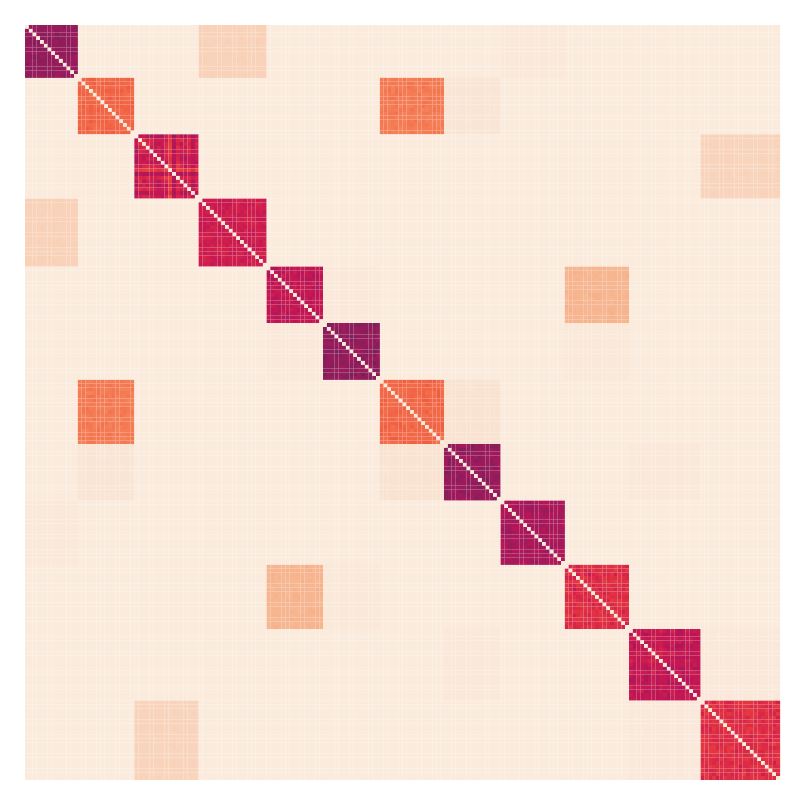}&
\hspace{-0.3cm}\includegraphics[width=0.2\columnwidth]{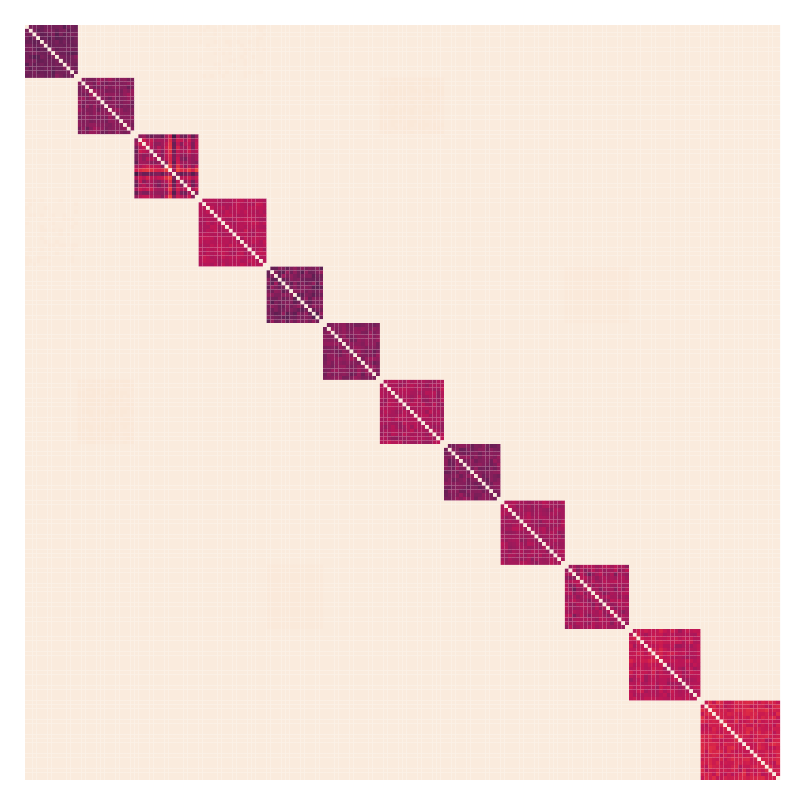}\\[-2pt]
\hspace{-0.2cm}{\footnotesize (e) Items relationship} & \hspace{-0.2cm} {\footnotesize (f) $\mL_c(t=5000)$} & \hspace{-0.2cm} {\footnotesize (g) $\mL_c(t=7000)$} & \hspace{-0.2cm} {\footnotesize (h) $\mL_c(t=10000)$}\\
\end{tabular}
\end{center}\vspace{-0.4cm}
\caption{\footnotesize The two panels in the first column shows the 
exact correlation among rows and columns, respectively. The other three columns are the Laplacian matrix learned by AIR at different iterations. 
}\label{fig:rela_recom}
\end{figure}

As shown in \Figref{fig:baboon}, 
both Laplacian matrices $\mL_r(t)$ and $\mL_c(t)$ 
appear many blocks during the initial training stage, e.g., $t=4000$. We consider a few boxed blocks of the matrix $\mL_c$ in \Figref{fig:baboon} (b, bottom); the values in these blocks reflect the strong similarity of the corresponding highlighted patches of the original Baboon image, which echos our intuition. Moreover, the slight difference between those selected columns will be future captured by the adaptive regularizer as the training goes, e.g., at $t=7000$. As the training goes further, the regularization gradually vanishes. We see that the regularization becomes rather weak at $t=10000$. Syn-Netflix experiments also support the above conclusions, $\mL_r$ and $\mL_c$ tend to the Users/Items relationship (in \Figref{fig:rela_recom} (a, e)). Different from the case in \Figref{fig:baboon}, the $\mL_r(t)$ and $\mL_c(t)$ in \Figref{fig:rela_recom} retain the value when the rating matrix in \Figref{fig:NMAEDuringTaining} (a) corresponding rows or columns are identical.

These results confirm that AIR captures the similarity from large to small. This raises an interesting question: \emph{does there exist an optimal $t^*$ such that $\mL_r(t^*)$ and $\mL_c(t^*)$ are precisely captured by AIR?
If yes, can we use these fixed optimal $\mL_r(t^*)$ and $\mL_c(t^*)$ for AIR and further train $\mX$ to obtain even better matrix completion?}

Next, we experimentally show that Laplacian matrices $\mL_r$ and $\mL_c$ learned by AIR are crucial for matrix completion. We fix $\mL_r$ and $\mL_c$ that were learned at a specific training step for AIR. In particular, we compare AIR with adaptive $\mL_r$ and $\mL_c$, and the variants of AIR with both Laplacian matrices fixed as that learned at $t=4000$, $7000$, and $9000$, respectively. We use the same hyperparameters for training AIR with fixed Laplacian matrices as that used for training AIR. To quantitatively compare the performance of AIR with different settings, we consider the following Normalized Mean Absolute Error (NMAE) to measure the gap between the recovered and exact elements at unobserved locations,
$$
\mathrm{NMAE}=\frac{\left\|\bar{\gA}(\mX)-\bar{\mY}\right\|_F^2}{(mn-o)\left(\mY^*_{\max}-\mY^*_{\min}\right)},
$$
where $\bar{\mY}$ is the unobserved elements, $\bar{\gA}$ maps $\mX$ to the corresponding unobserved location, and $\mY^*=\left[\mY,\bar{\mY}\right]^\top$ is the ground truth.

We contrast the vanilla AIR and AIR with fixed Laplacian matrices for Baboon image inpainting. \Figref{fig:DiffEpoch} shows how the NMAE changes during training.
AIR, which updates the regularization during training, achieves the best performance for all missing patterns. Fixing Laplacian matrices can accelerate the convergence of training AIR at the beginning. 
While the optimal time $t^*$ is unknown for AIR with fixed Laplacian matrices, the learned Laplacian matrices work without estimating $t^*$.

\begin{figure}
\begin{center}
\begin{tabular}{ccc}
\hspace{-0.2cm}\includegraphics[width=0.29\columnwidth]{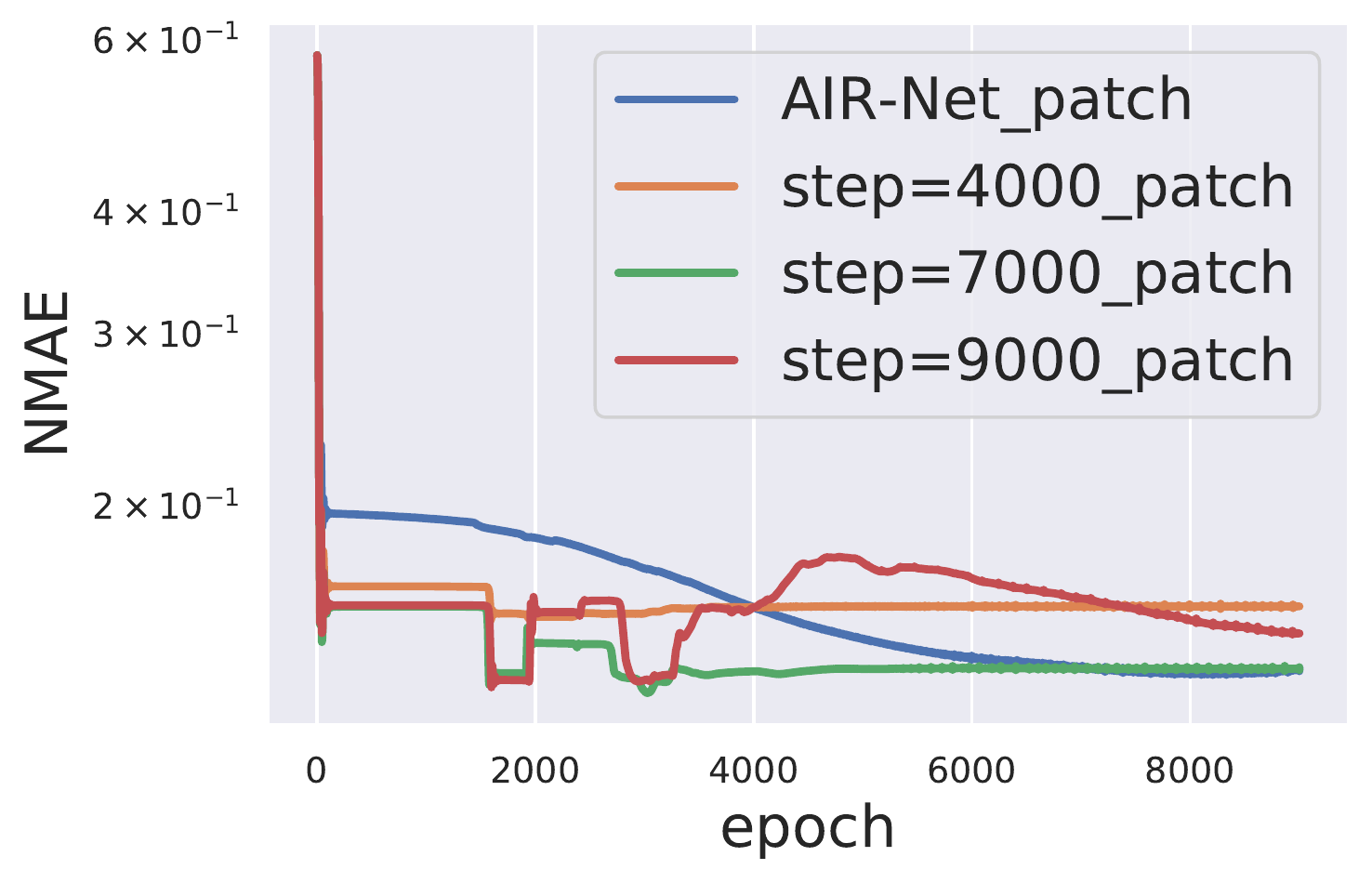}&
\hspace{-0.3cm}\includegraphics[width=0.29\columnwidth]{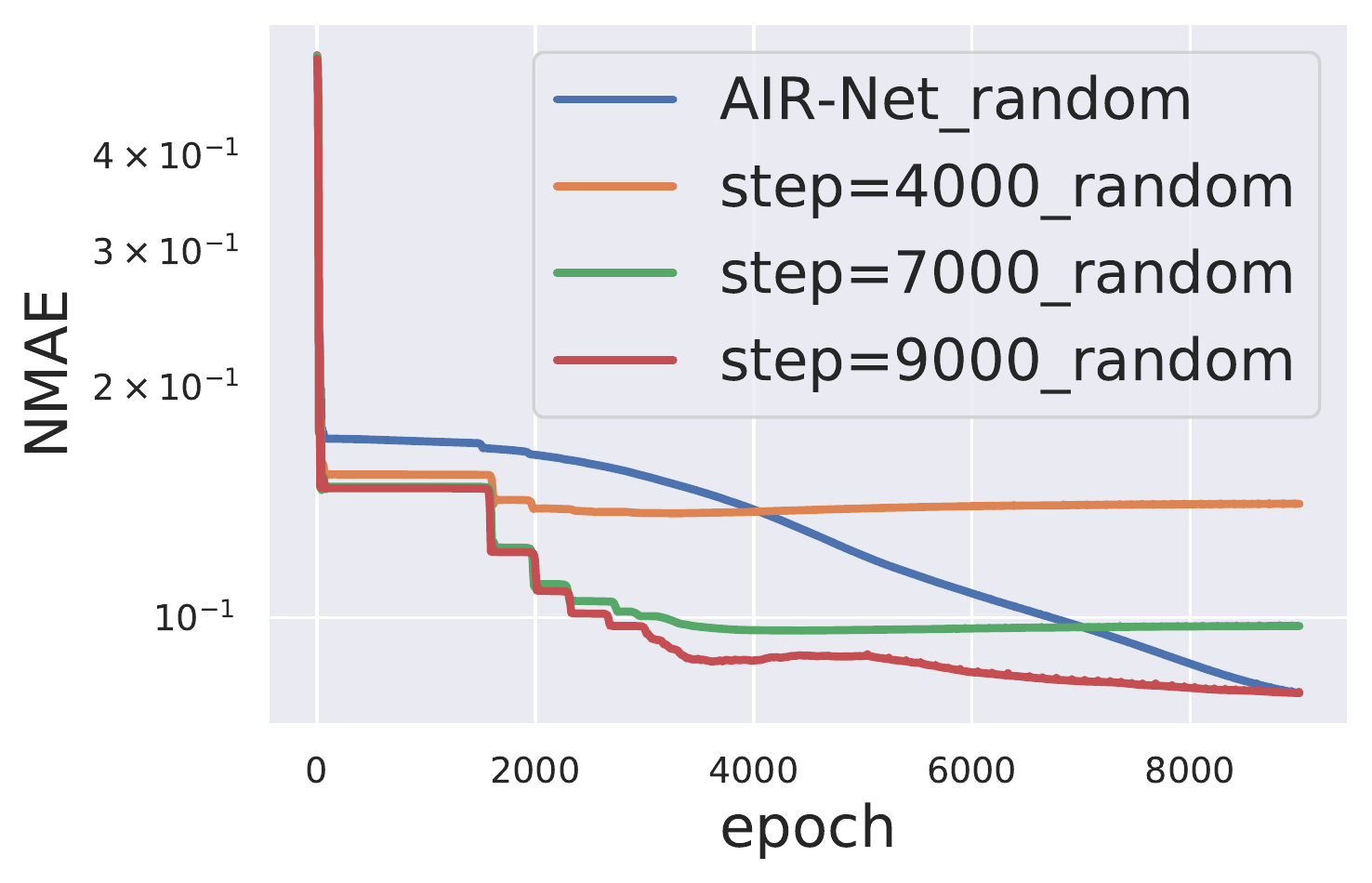}&
\hspace{-0.3cm}\includegraphics[width=0.29\columnwidth]{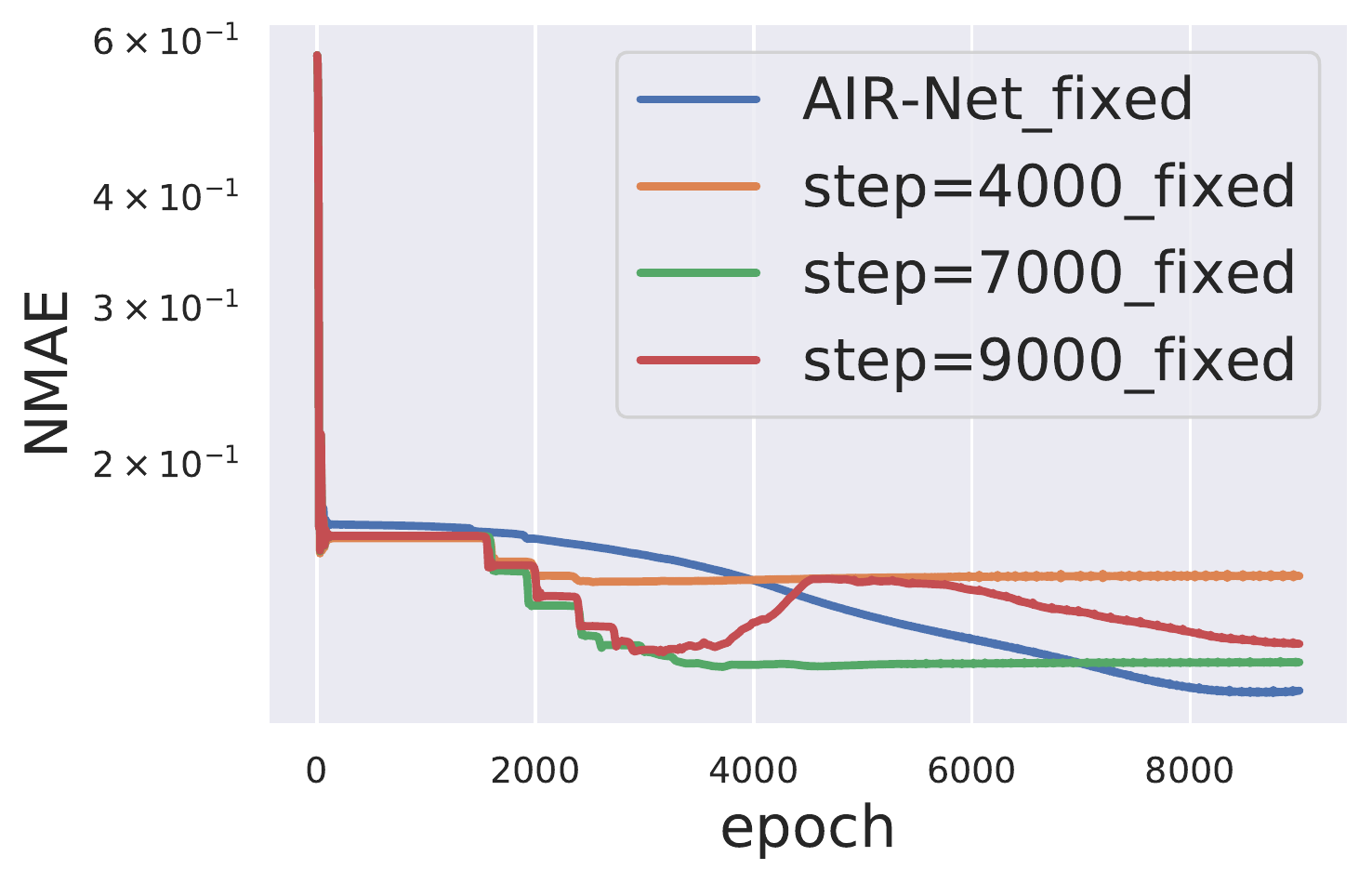}\\[-0.2cm]
{\footnotesize (a) Patch} & {\footnotesize (b) Random} & {\footnotesize (c) Fixed} \\
\end{tabular}
\end{center}\vspace{-0.4cm}
\caption{\footnotesize Contrasting the adaptive regularizer with the fixed regularizer for Baboon image recovery. We consider inpainting the Baboon image with three types of missing pixels: (a) patch missing, (b) randomly missing 30\% pixels, and (c) texture missing.
The blue lines plot the NMAE during training the vanilla AIR. The remaining three lines in each figure indicate replacing $\mL_r$ and $\mL_c$ with $\mL_r(t)$ and $\mL_c(t)$ at $t=4000$, $7000$, and $9000$, respectively.
}\label{fig:DiffEpoch}
\end{figure}

\subsection{AIR behaves like a momentum}
This subsection makes a heuristic and direct connection between AIR and the momentum method. 
AIR converges to the vanilla DMF model when the adaptive regularization vanishes, and Corollary \ref{cor..ConvReg} shows that the adaptive regularization vanishes exponentially fast. We need to distinguish AIR and DMF according to the optimization trajectory. 
We compare three models (DMF, DMF+TV, and DMF+AIR, i.e., AIR) in \Figref{fig:TraDMFVar}, and we see that the three models perform dramatically different near the convergence. At the beginning of training, the observed and unobserved MSEs of the three models drop similarly. 
When the observed MSE becomes smaller, the model learns 
details in observed elements. The unobserved MSE increased during the observed MSE decrease in the vanilla DMF and DMF+TV cases. Our proposed AIR keeps the decaying trend for both observed and unobserved MSEs.

Looking back into the training of AIR, we see that the update of $\mX(t+1)$ involves both $\mX(t)$ and $\mL(t)$, and the update of $\mL(t)$ depends on $\mX(t-1)$.
To understand the training dynamics, we consider the following simplified model 
$$
\Min_{\mX,\mL}\left\{\gL=\gL_{\mathbb{Y}}+\text{tr}\left(\mX^\top\mL\mX\right)\right\},
$$
and we have 
$$
\nabla_{\mX(t)}\gL=\nabla_{\mX(t)}\gL_{\mathbb{Y}}+2\mL(t)\mX(t),
$$
where $\mL(t)$ is the function of $\left\{\mX(t_0)\mid t_0<t\right\}$. Therefore, every iteration step of AIR 
leverages all the learned previous information $\left\{\mX(t_0)\mid t_0<t\right\}$. While the update of both vanilla DMF and DMF+TV only depends on $\mX(t)$.
From this viewpoint, AIR shares a similar spirit as the momentum method, which leverages history to improve performance.

\begin{figure}
\begin{center}
\begin{tabular}{ccc}
\hspace{-0.2cm}\includegraphics[width=0.29\columnwidth]{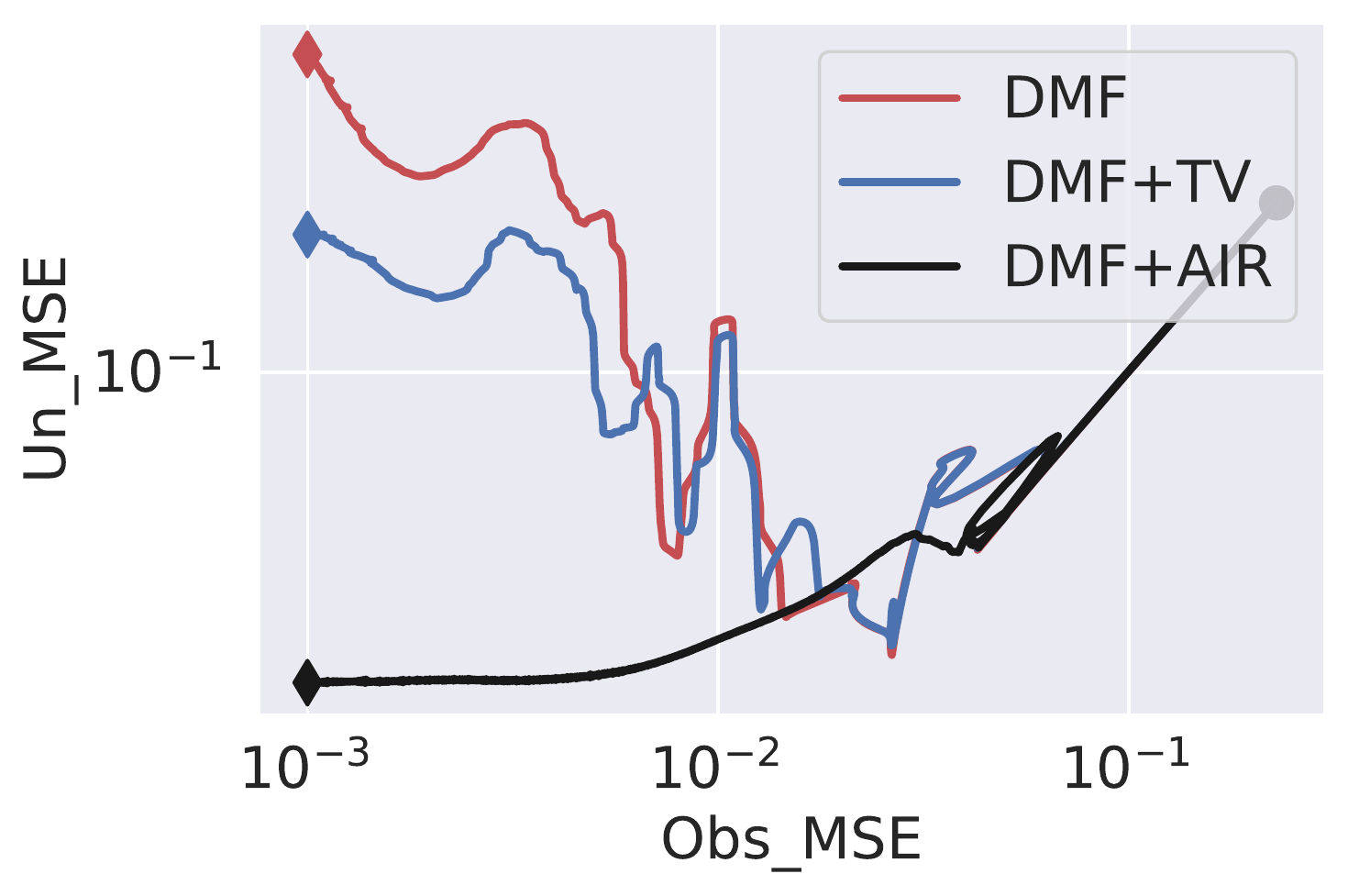}&
\hspace{-0.3cm}\includegraphics[width=0.29\columnwidth]{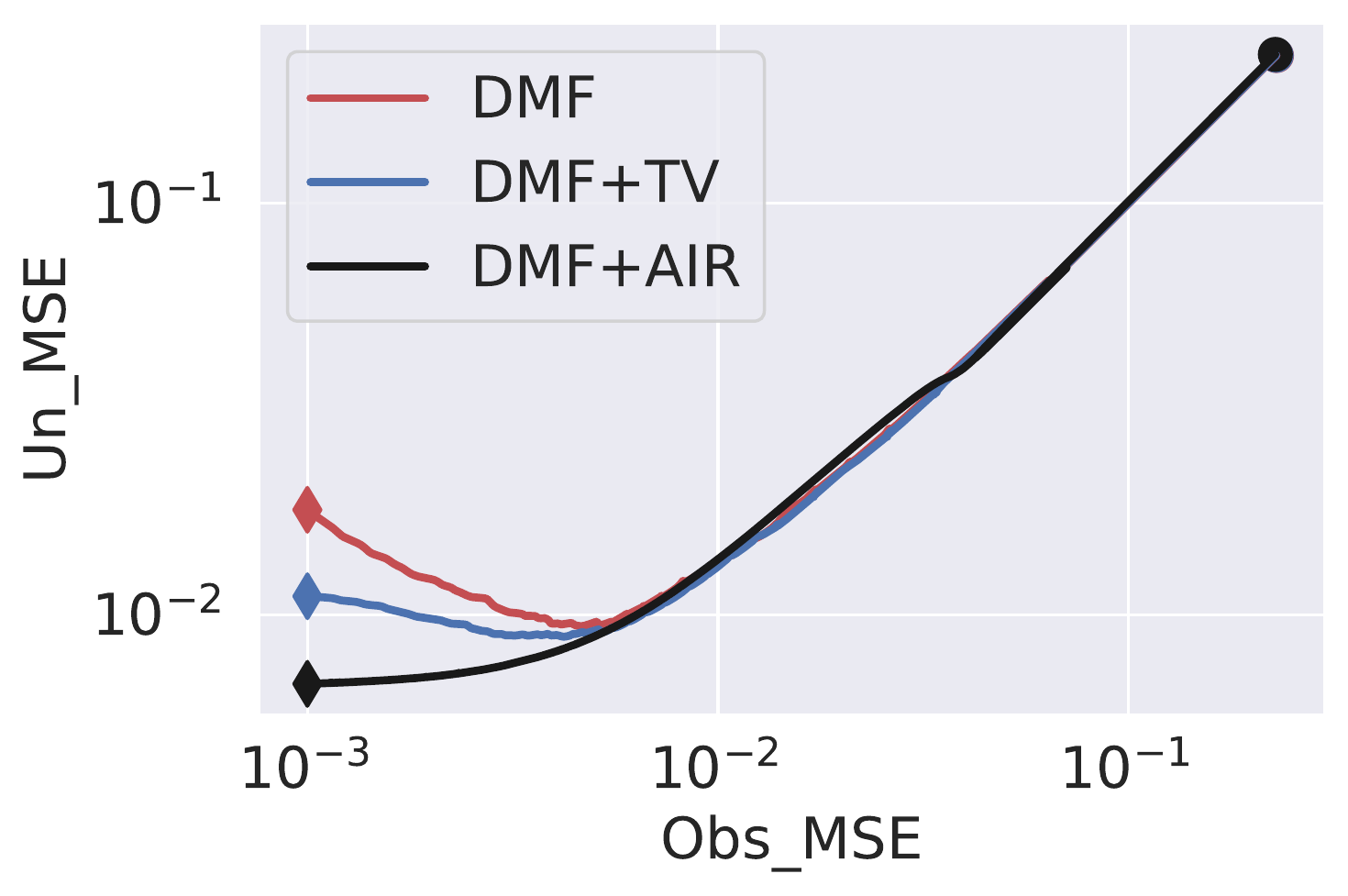}&
\hspace{-0.3cm}\includegraphics[width=0.29\columnwidth]{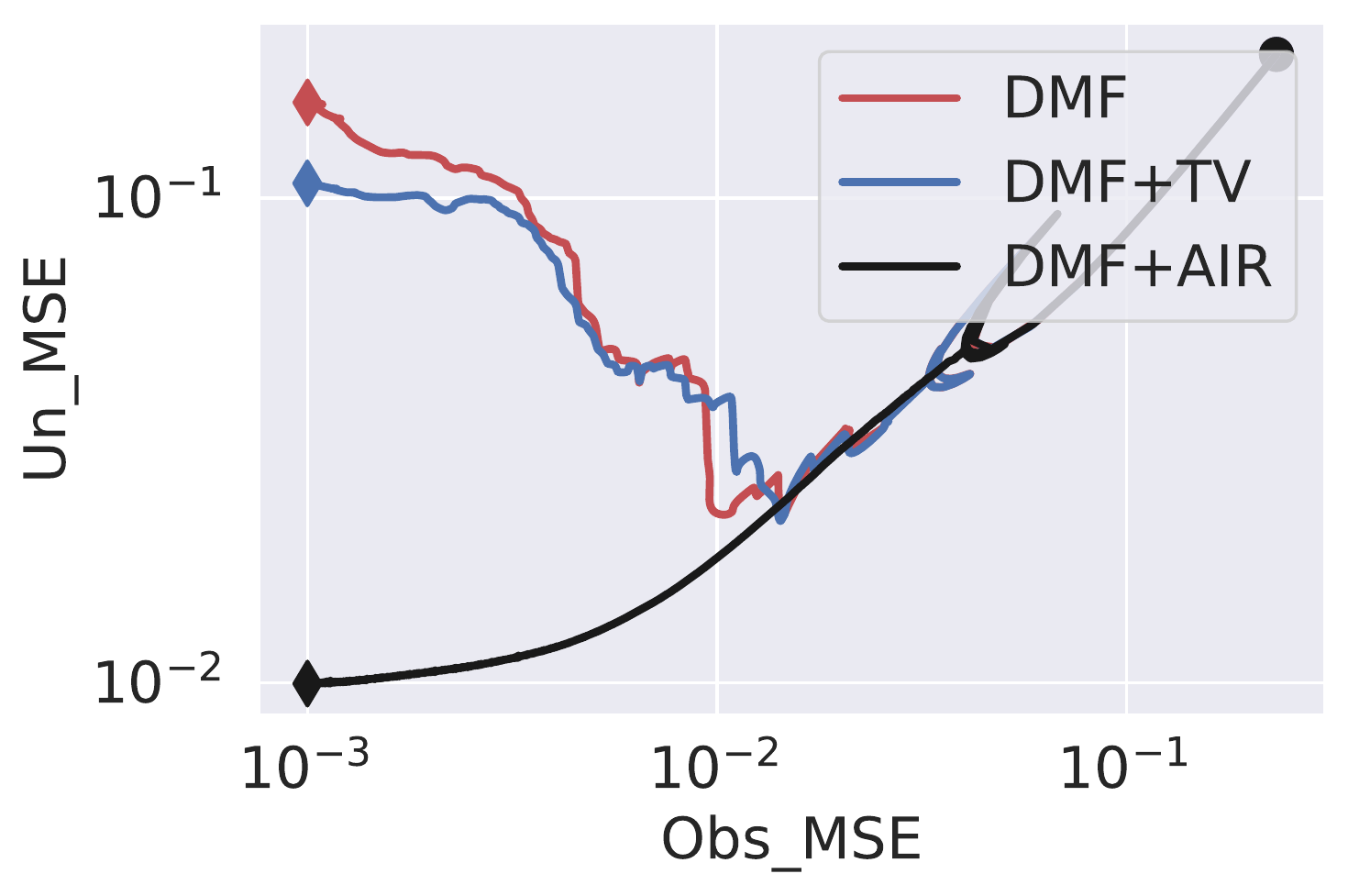}\\[-4pt]
{\footnotesize (a) Patch} & {\footnotesize (b) Random} & {\footnotesize (c) Fixed} \\
\end{tabular}
\end{center}\vspace{-0.4cm}
\caption{\footnotesize Trajectories of training DMF model without regularization, with TV, and with AIR regularization for inpainting the Baboon image with three types of missing pixels: (a) patch missing, (b) randomly missing 30\% pixels, and (c) texture missing. The dot point indicates $X(0)$ and the diamond shape of different color indicate $X(10000)$ for different models.
}\label{fig:TraDMFVar}
\end{figure}

\subsection{Adaptive regularizer performance on varied data and missing pattern}\label{sec..performance}
We apply AIR for matrix completion on three data types with different missing patterns. We will also compare AIR with a few baseline algorithms.  \Figref{fig:rec_drug} plots the raw matrices 
of Syn-Netflix and DTI datasets.

\textbf{Baseline algorithms.}
In the following experiments, we compare AIR with several popular matrix completion algorithms, including 
KNN \cite{Goldberger2004NeighbourhoodCA}, SVD \cite{Troyanskaya2001MissingVE}, PNMC \cite{Yang2020ANP}, DMF \cite{Arora2019ImplicitRI}, and RDMF \cite{Li2020ARD}.
For Syn-Netflix-related experiments, we replace RDMF with DMF+DE \cite{Boyarski2019SpectralGM}, which is a better fit for that task.

\textbf{Avoid Over-fitting.} \Figref{fig:NMAEDuringTaining} shows the evolution of training NMAE of DMF and AIR.
AIR avoids over-fitting and performs better on all three data types with various missing patterns compared with vanilla DMF.

\begin{figure}
\begin{center}
\begin{tabular}{ccc}
\hspace{-0.2cm}\includegraphics[width=0.29\columnwidth]{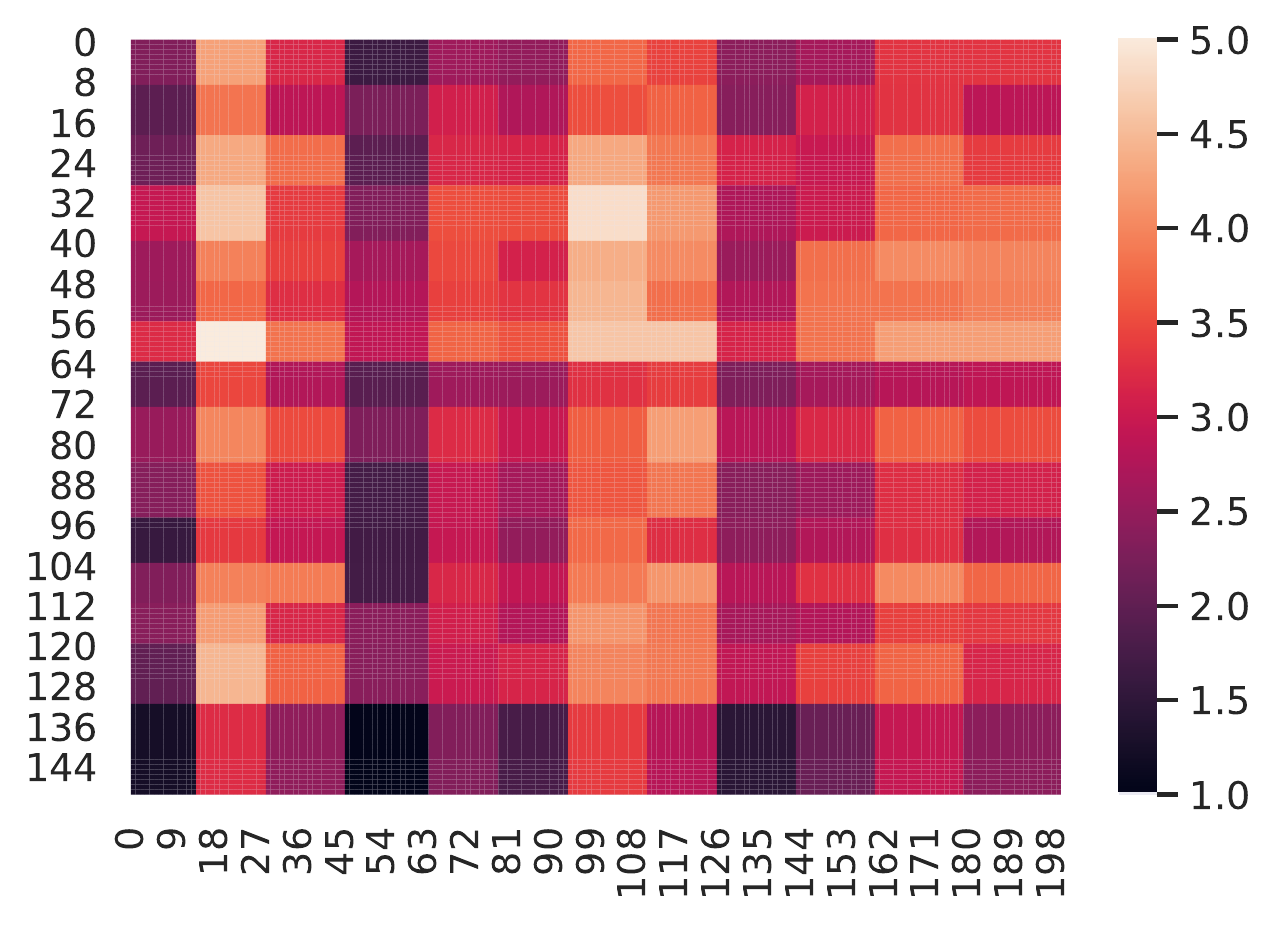}&
\hspace{-0.3cm}\includegraphics[width=0.29\columnwidth]{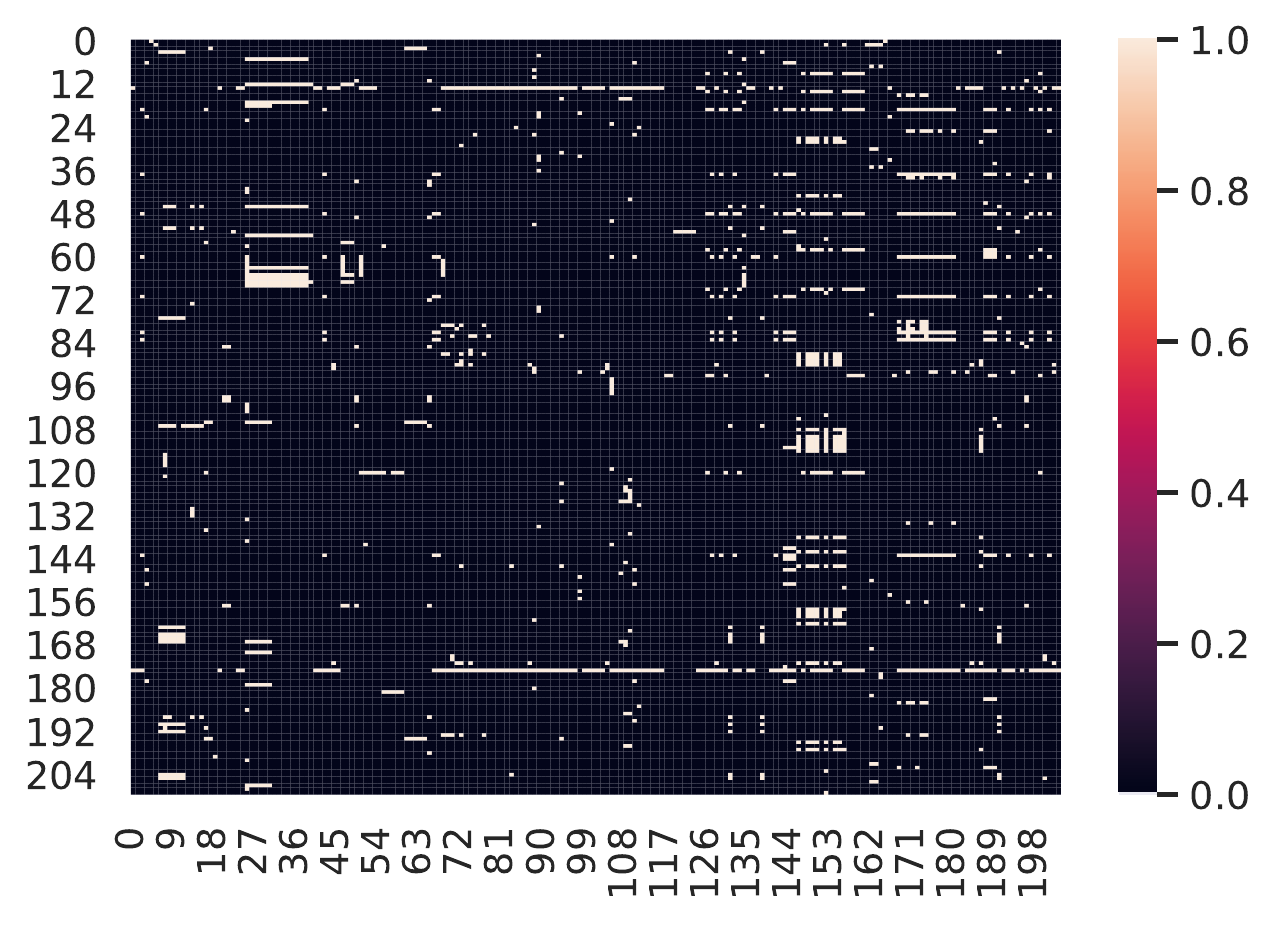}&
\hspace{-0.3cm}\includegraphics[width=0.29\columnwidth]{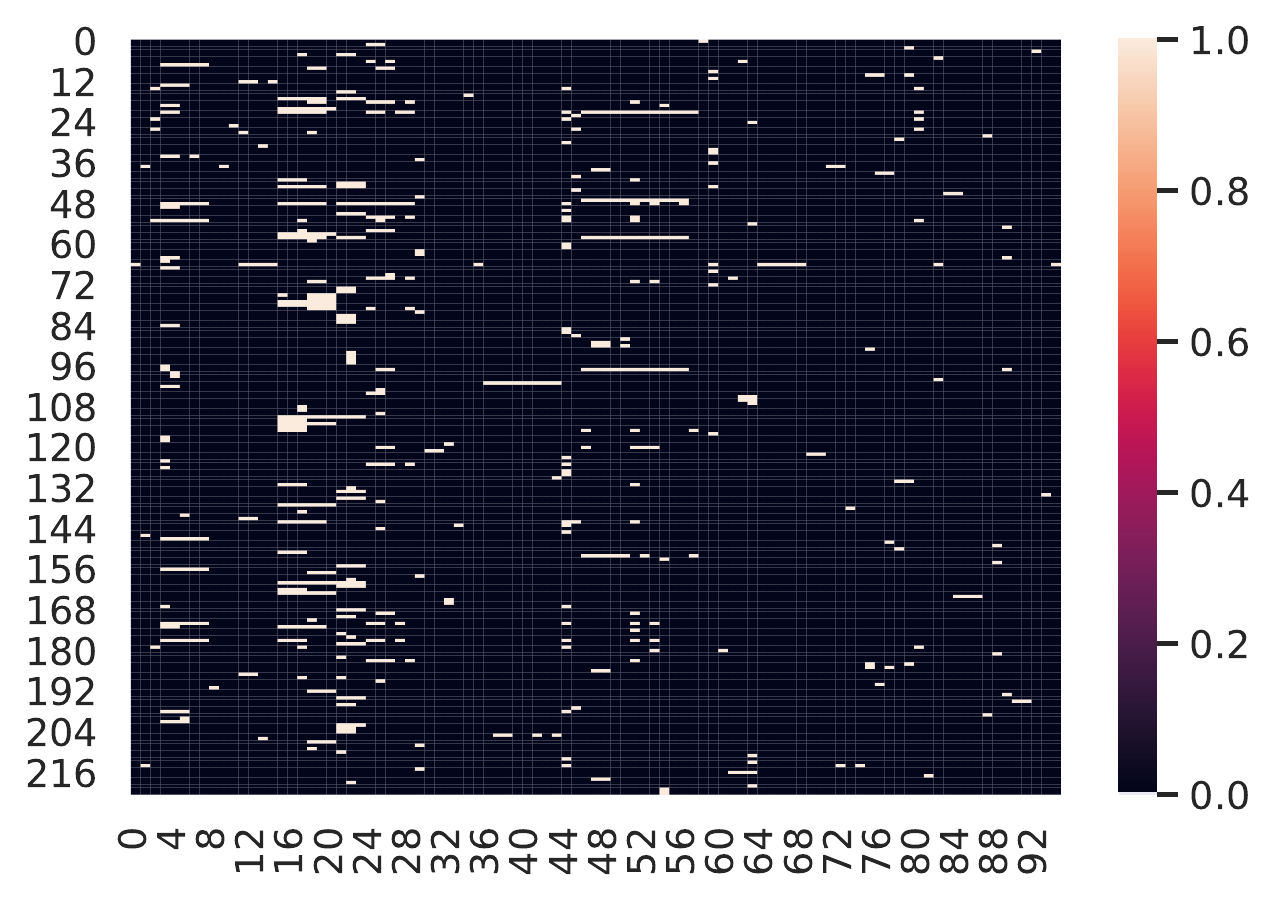}\\[-4pt]
{\footnotesize (a) Syn-Netflix} & {\footnotesize (b) IC} & {\footnotesize (c) GPCR} \\
\end{tabular}
\end{center}\vspace{-0.4cm}
\caption{\footnotesize Matrix completion 
for several benchmark datasets, including Syn-Netflix, IC, and GPCR. Deeper color in Syn-Netflix heatmap indicate a higher rate from a user to a given item. IC and GPCR are both drug-target interaction (DTI) \cite{Mongia2020DrugtargetIP} matrix, and the entries of the matrix represent the interaction between drug and target with deeper color represents stronger indication. We consider binary interaction for both IC and GPCR.
}\label{fig:rec_drug}
\end{figure}

\begin{figure}
\begin{center}
\begin{tabular}{ccc}
\hspace{-0.2cm}\includegraphics[width=0.3\columnwidth]{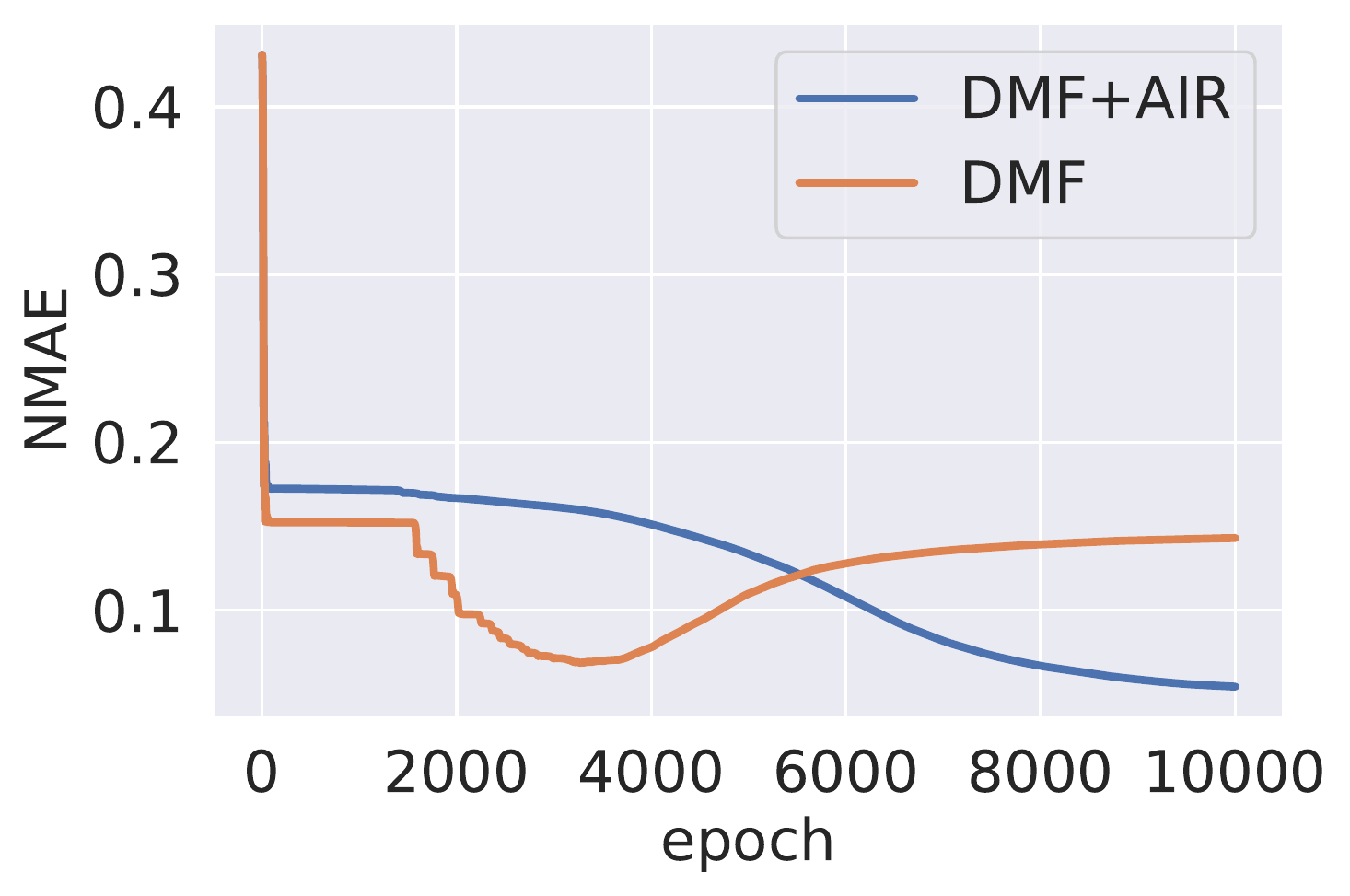}&
\hspace{-0.3cm}\includegraphics[width=0.3\columnwidth]{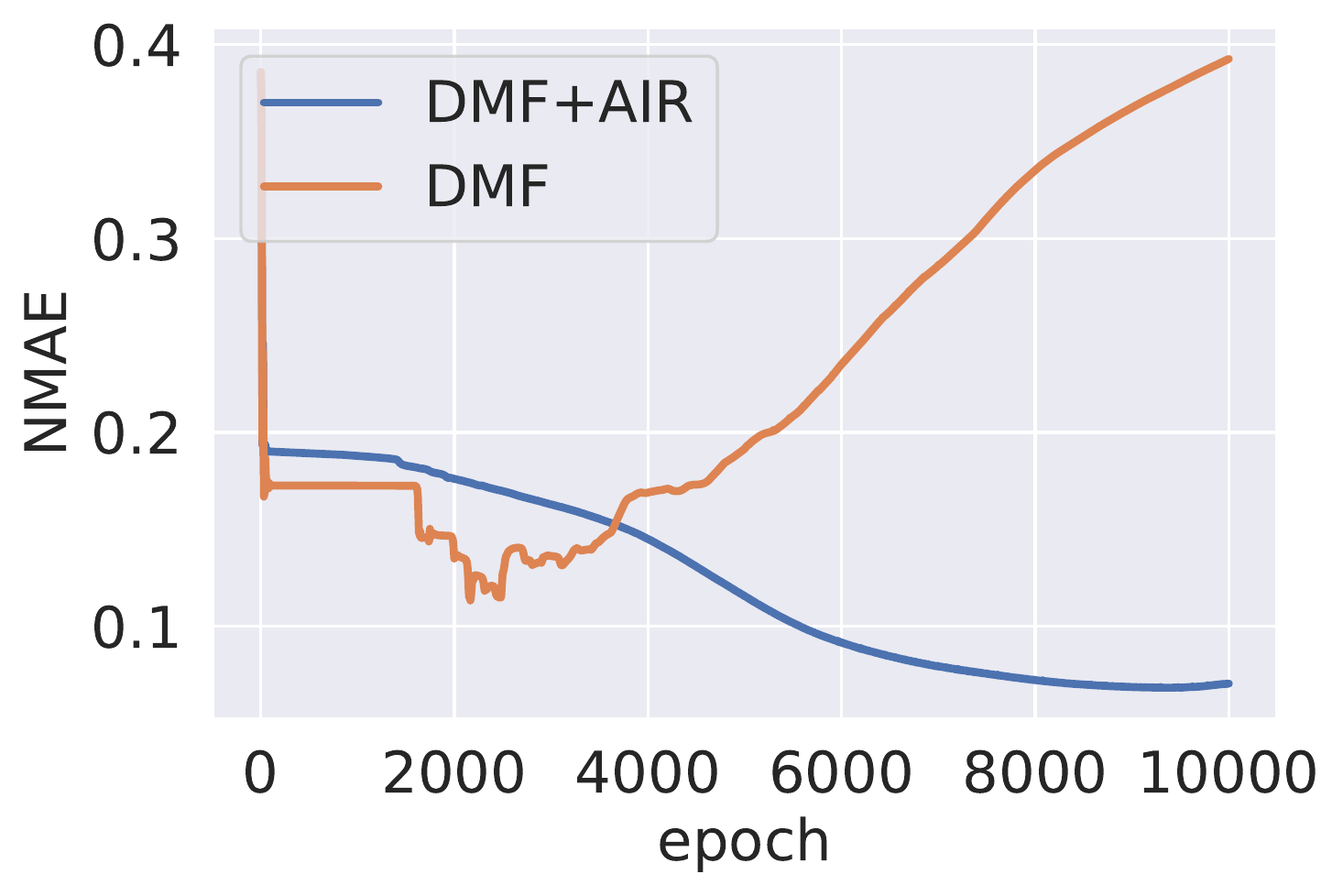}&
\hspace{-0.3cm}\includegraphics[width=0.3\columnwidth]{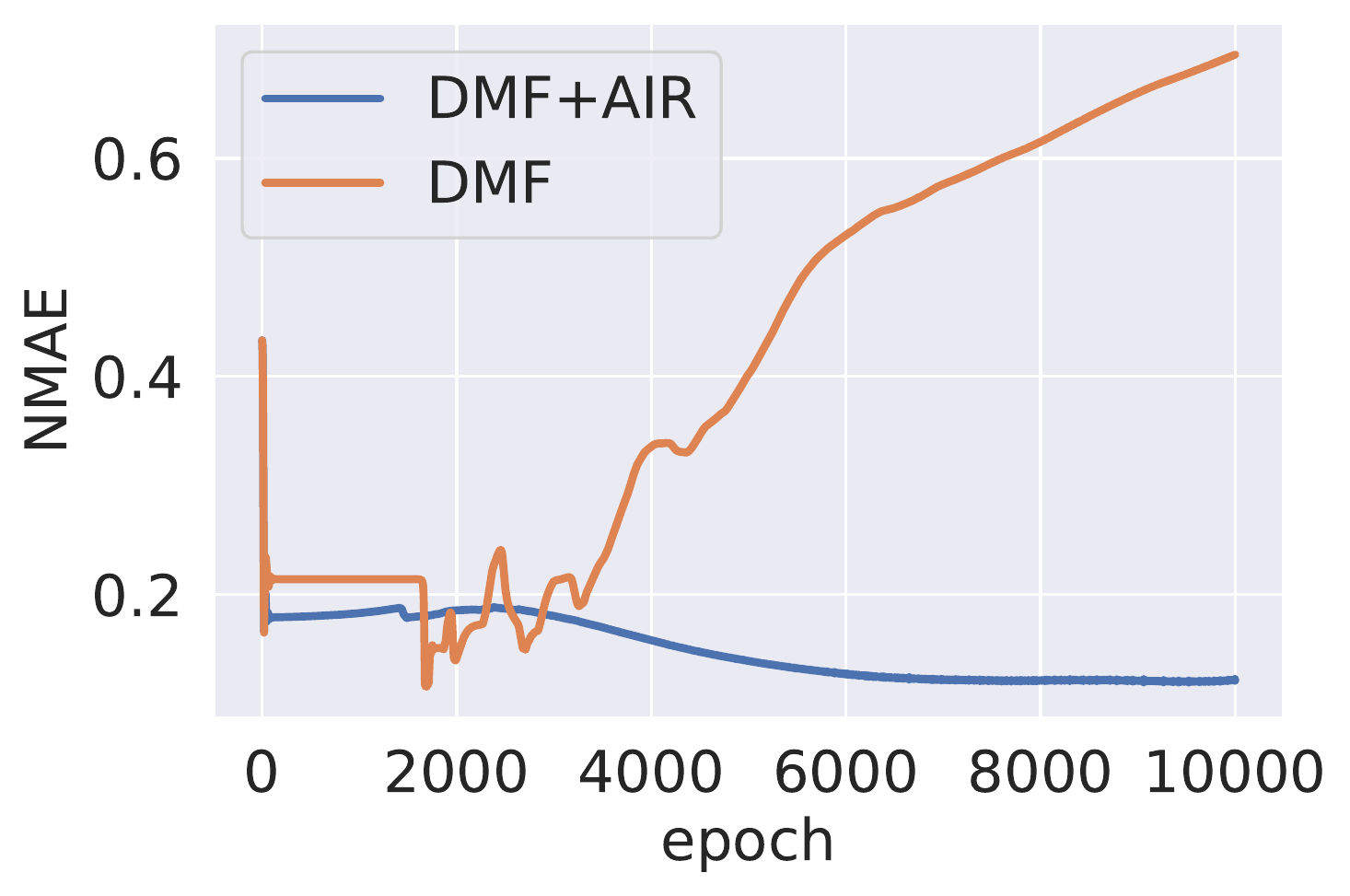}\\[-4pt]
{\footnotesize (a) Random Cameraman} & {\footnotesize (b) Textural Cameraman} & {\footnotesize (c) Patch Cameraman} \\
\hspace{-0.2cm}\includegraphics[width=0.3\columnwidth]{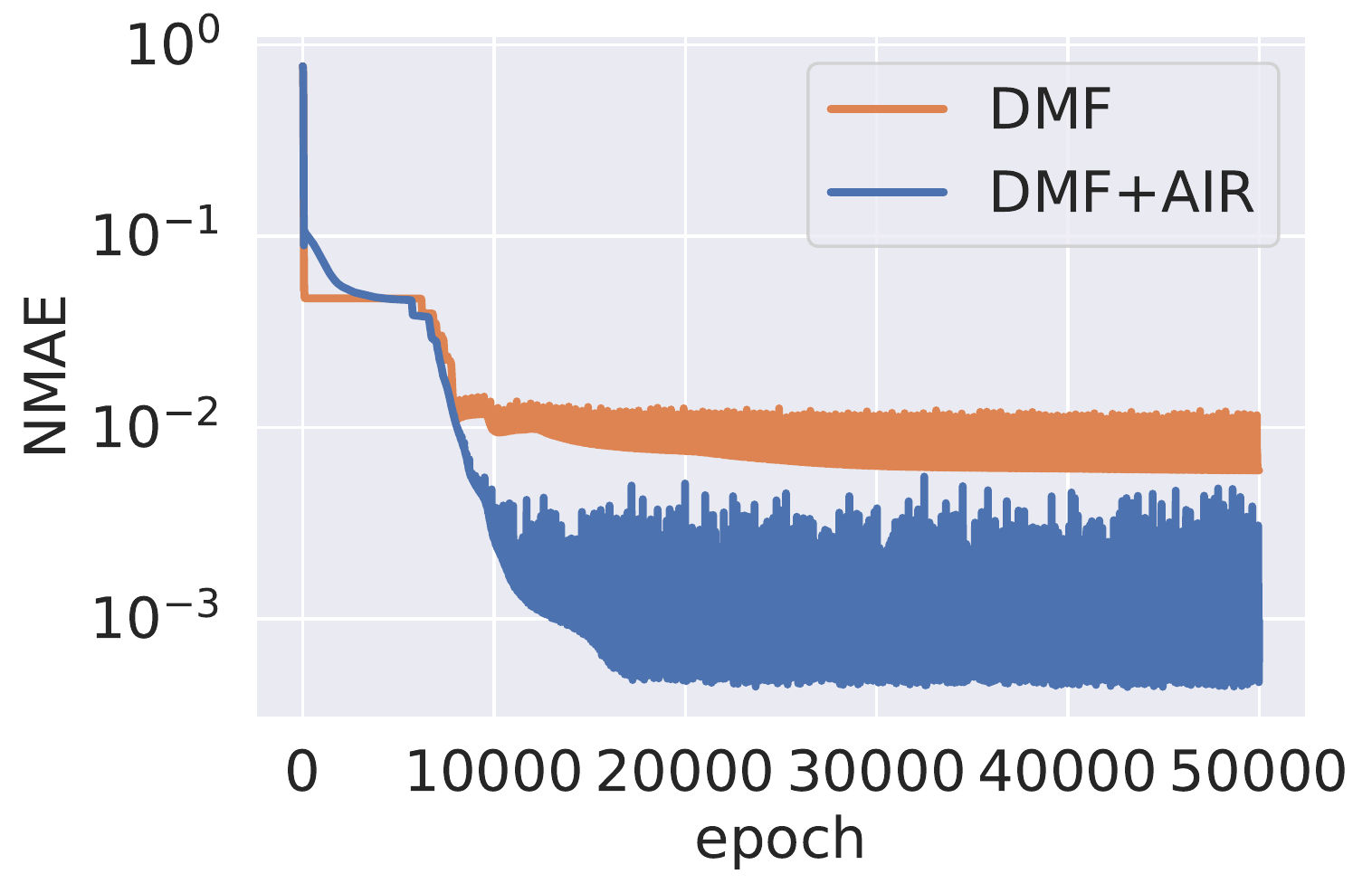}&
\hspace{-0.3cm}\includegraphics[width=0.3\columnwidth]{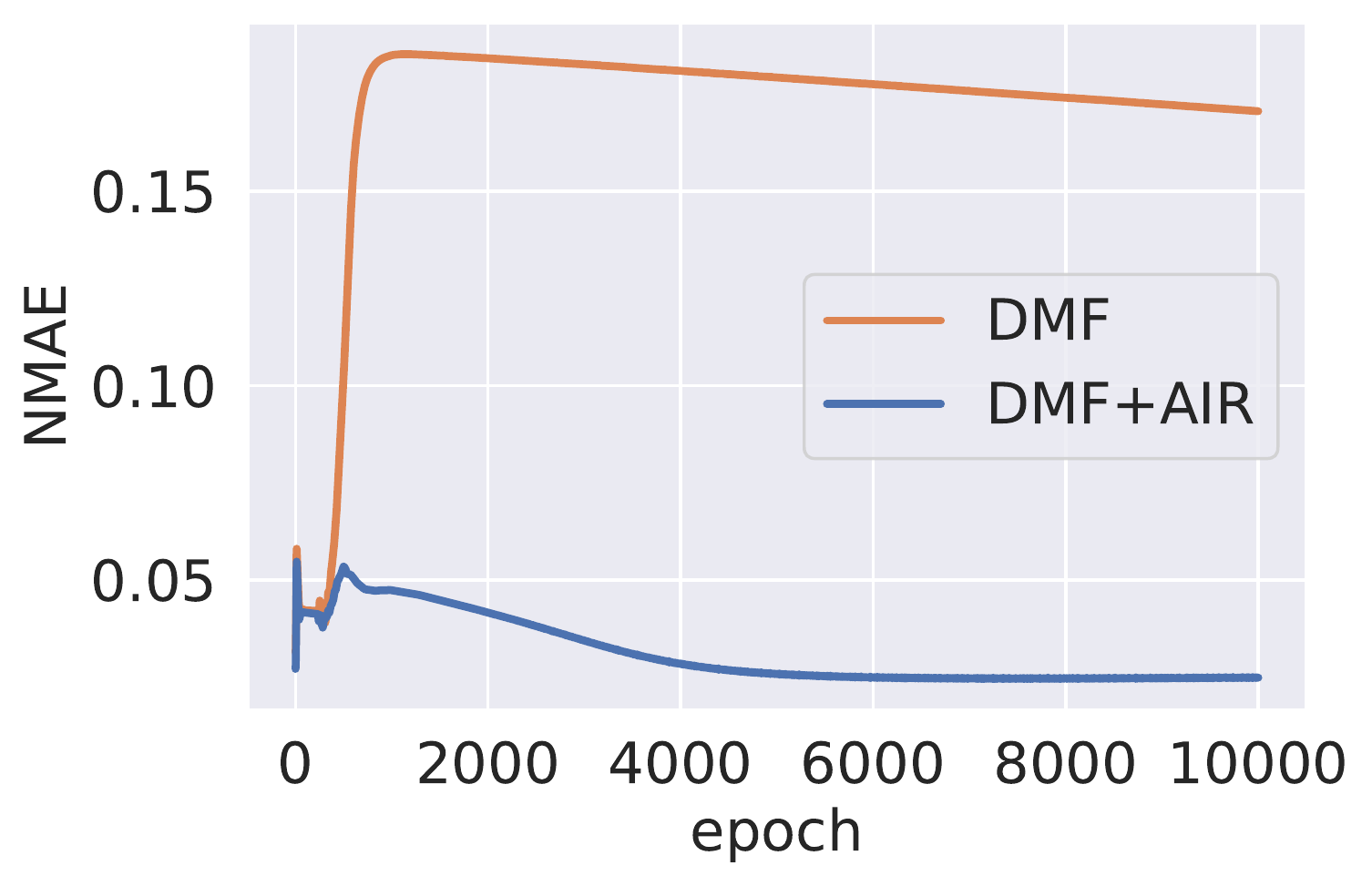}&
\hspace{-0.3cm}\includegraphics[width=0.3\columnwidth]{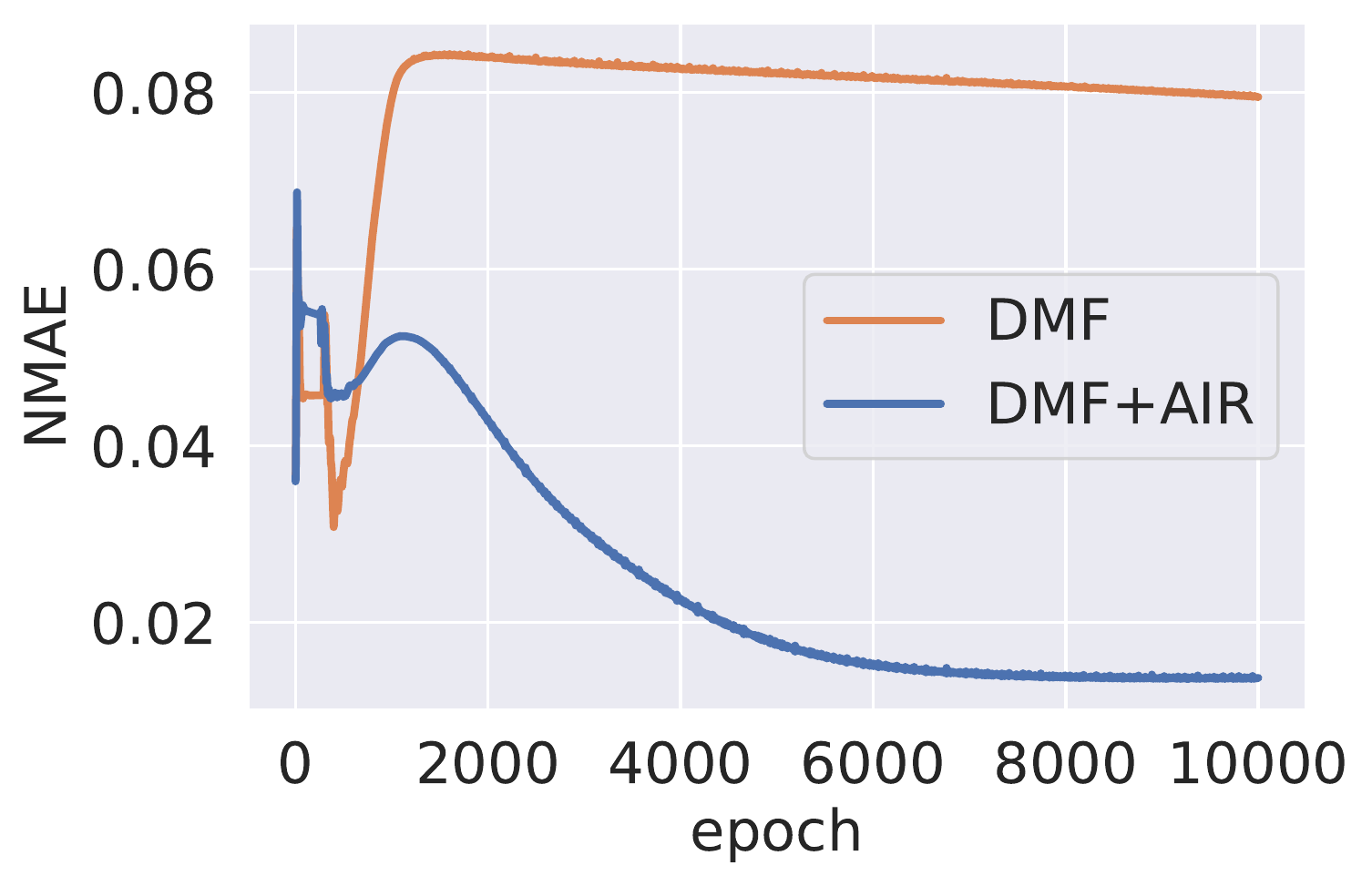}\\[-4pt]
{\footnotesize (d) Random Syn-Netflix} & {\footnotesize (e) Random IC} & {\footnotesize (f) Random GPCR} \\
\end{tabular}
\end{center}\vspace{-0.4cm}
\caption{\footnotesize Evolution of NMAE during training DMF and AIR (DMF with AIR regularization). 
The first row plots the results of Cameraman with (a) random missing, 
(b) textural missing, and (c) patch missing. 
The second row plots three other benchmarks 
with random missing:
(d) Syn-Netflix, (e) IC, and (f) GPCR
} \label{fig:NMAEDuringTaining}
\end{figure}

\begin{table}[htp]
\centering
\caption{
NMAE values of different algorithms for solving different matrix completion problems with different missing patterns.
NA indicates that the method is not suitable for that task.
The hyper-parameters of models and algorithms keep consistent with the original paper.}
\label{tab:multidata}
\begin{tabular}{cccccccccc}
\hline
\multicolumn{1}{l}{\scriptsize Data}  & {\scriptsize Missing} & {\scriptsize KNN\cite{Goldberger2004NeighbourhoodCA}}   & {\scriptsize SVD\cite{Troyanskaya2001MissingVE}}   & {\scriptsize PNMC\cite{Yang2020ANP}}  & {\scriptsize DMF\cite{Arora2019ImplicitRI}}   & {\scriptsize RDMF\cite{Li2020ARD}}    &{\scriptsize DMF+DE\cite{Boyarski2019SpectralGM}}       & {\scriptsize AIR}       \\ \hline
\multirow{3}{*}{\scriptsize Barbara}   & $30\%$  & 0.083 & 0.0621 & 0.0622 & 0.0613 & 0.0494  &     NA  & \textbf{0.0471} \\
                           & Patch   & 0.1563 & 0.2324 & 0.2055 & 0.7664 & 0.3025  &     NA   & \textbf{0.1195} \\
                           & Texture & 0.0712  & 0.1331 & 0.1100 & 0.3885 & 0.1864  &    NA    & \textbf{0.0692} \\
\hline
\multirow{3}{*}{\scriptsize Baboon}   & $30\%$  & 0.0831 & 0.1631 & 0.0965 & 0.2134 & 0.0926   &   NA    & \textbf{0.0814} \\
                           & Patch   & 0.1195  & 0.1571 & 0.1722 & 0.8133 & 0.2111  &   NA     & \textbf{0.1316} \\
                           & Texture & 0.1237  & 0.1815 & 0.1488 & 0.5835 & 0.2818  &   NA     & \textbf{0.1208} \\ 
\hline
\multirow{3}{*}{\scriptsize Syn-Netflix}   &$70\%$ & 0.0032 & 0.0376  & NA & 0.0003 &    NA  & 0.0008    & \textbf{0.0002} \\
                                &$75\%$ & 0.0046 & 0.0378  & NA & 0.0004 &    NA   & 0.0009    & \textbf{0.0003} \\
                                &$80\%$ & 0.0092  & 0.0414   & NA & 0.0014   &    NA   & 0.0012   & \textbf{0.0007} \\
\hline
{\scriptsize IC}                & $20\%$ &  0.0169    &  0.0547    & NA   &   0.0773   & NA &   0.0151      &     \textbf{0.0134}  \\
{\scriptsize GPCR}                         &  $20\%$  &   0.0409    &   0.0565      &  NA   &  0.1513  &  NA  &    \textbf{0.0245}       &    0.0271     \\
                           \hline
\end{tabular}
\end{table}

\textbf{Adaptive to data.} Table~\ref{tab:multidata} lists the NMAEs of matrix completion using AIR and several benchmark algorithms for different data with different missing patterns. We see that, in general, AIR performs the best and even better than the well-calibrated algorithms for some particular datasets. We further plot the recovered images in \Figref{fig:rec_img}; which shows that AIR consistently gives appealing results visually. Some baseline algorithms also perform well for some specific missing patterns. For instance, RDMF performs well for the random missing case but poorly for the other missing patterns. PNMC can complete the missing patch well but does not work well when the texture is missing. Overall, AIR achieves decent results visually and quantitatively for all data under all missing patterns.

\begin{figure}[htp]
    \hspace{1cm}\includegraphics[width=0.9\linewidth]{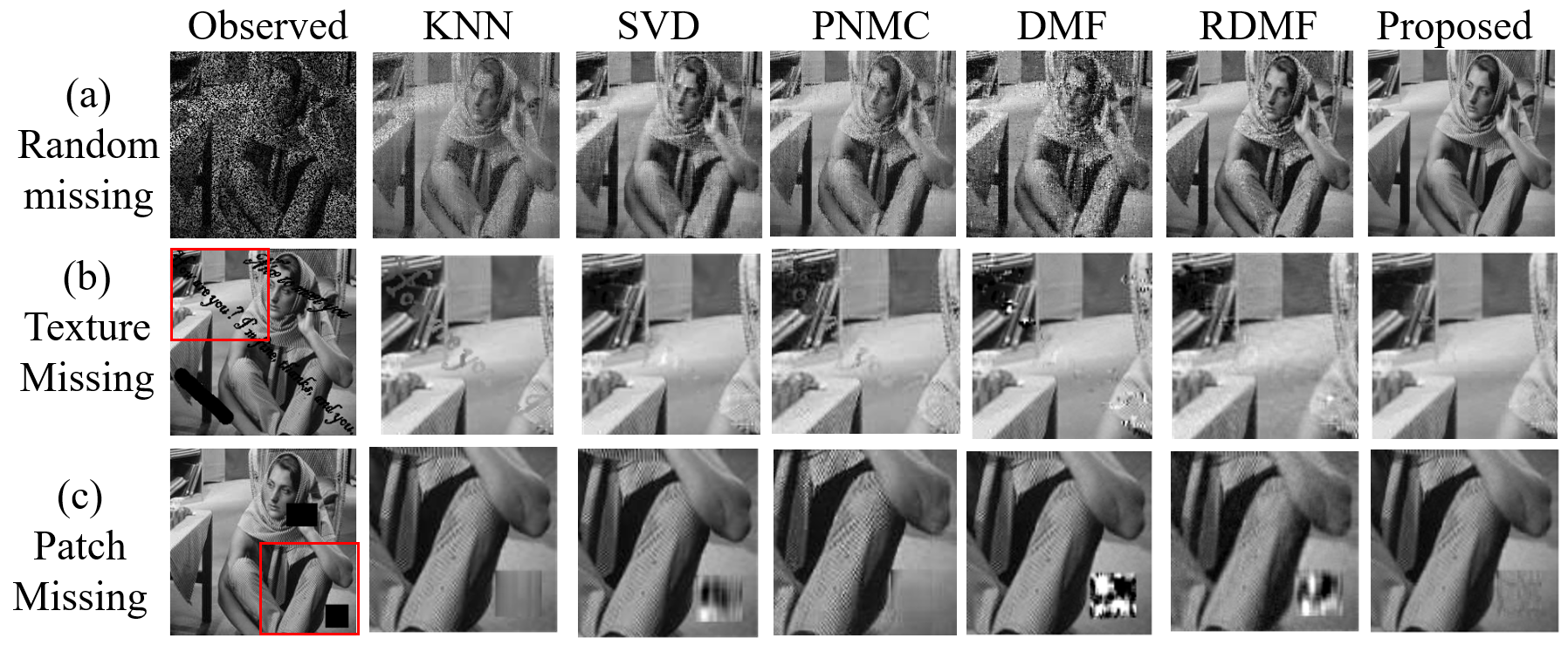}\vspace{-0.3cm}
    \caption{\footnotesize Contrasting KNN \cite{Goldberger2004NeighbourhoodCA},  SVD \cite{Troyanskaya2001MissingVE}, PNMC \cite{Yang2020ANP}, DMF \cite{Arora2019ImplicitRI}, RDMF \cite{Li2020ARD}, and our proposed AIR on the Babara image inpainting with three types of missing data respectively. The hyper-parameters of benchmark models and algorithms are adopted from the original paper.
    }
    \label{fig:rec_img}
\end{figure}

\subsection{Contrasting computational time of AIR with DMF}
The proposed AIR has several advantages, including avoiding over-fitting and adapting to data. However, 
AIR has more parameters than vanilla DMF, requiring 
more computational time in 
each iteration. In this subsection, we compare the complexity of AIR with two benchmark algorithms, namely, DMF and RDMF.

We focus on square matrices of the size $m\times m$. 
The loss function of $L$-layer DMF is given by $\frac{1}{2}\left\|\gA\left(\prod_{l=0}^{L-1}\mW^{[l]}\right)-\mY^*\right\|_F^2$, where $\mW^{[l]}\in\mathbb{R}^{m\times m}$ for $l=0,\ldots,L-1$ and $\mX = \prod_{l=0}^{L-1}\mW^{[l]}$ is the recovered matrix. Therefore the computational complexity of computing the loss function of DMF is $\gO((L-1)m^3)$. As for RDMF, its loss function involves the calculation of an additional TV term whose computational complexity is  $\gO(m^2)$. Therefore, RDMF is only slightly more expensive than DMF. 
The AIR's loss function calculates two additional adaptive regularization terms $tr(\mX\mL_r\mX^T)$ and $tr(\mX^T\mL_c\mX)$, with computational complexity being $\gO(2m^3)$.

The previous analysis shows that the computational complexity of each model is dominated by the size of $\mX$ ($m$) and the depth of the model ($L$). We numerically verify the previous computational complexity analysis on the Barbara image with different $m$ and $L$. In particular, we iterate each model for $10000$ iterations and report the computational time in Table~\ref{tab:comtime}. We conduct our experiments on Tesla V100 GPUs and report the GPU time consumption of three kinds of DMF-based methods, where the reported GPU times are averaged over ten independent runs. The numerical results show that:
\begin{enumerate}
\item For any fixed $m$ and $L$, AIR takes more computational time than RDMF, and DMF is computationally the cheapest among the three models.

\item For any fixed $m$, as $L$ increases, the ratio between the computational time of AIR and DMF decays. This is because the computational complexity of DMF increases as DMF becomes deeper. Nevertheless, the computational complexity of both TV and adaptive regularization is independent of the depth $L$.

\item For any fixed depth $L$, as $m$ increases, the ratio between the computational time of RDMF and DMF decays, while the ratio between the computational time of AIR and DMF keeps near a constant. This result echoes our computational complexity analysis. In particular, the computational complexity ratio between RDMF and DMF is $\gO(1+\frac{1}{(L-1)m})$ and the ratio between AIR and DMF is $\gO(1+\frac{2}{(L-1)})$.
\end{enumerate}

    
    

Although AIR takes more computational time than the baseline algorithms, it achieves significantly better accuracy than the baselines, which is valuable in many applications.

\begin{table}[htp]
\caption{Contrasting computational time of DMF, RDMF, and AIR with different $m,L$. RDMF/DMF represents the ratio of the computational time of RDMF and DMF. AIR/DMF represents the 
ratio of the computational time of AIR and DMF. Unit: second.}
\label{tab:comtime}
\centering
\begin{tabular}{ccccccc}
\hline
\multicolumn{1}{l}{$m$}  & {$L$} & {\scriptsize DMF\cite{Arora2019ImplicitRI}}   & {\scriptsize RDMF\cite{Li2020ARD}} &{\scriptsize RDMF/DMF} & {\scriptsize AIR}&{\scriptsize AIR/DMF}       \\ \hline
\multirow{3}{*}{100}  &2&31.55&44.24&1.40&59.53&1.89 \\
&3&33.93&46.81&1.38&61.53&1.81 \\
&4&36.52&49.16&1.35&61.59&1.69 \\
\hline
\multirow{3}{*}{170}  &2&35.15&47.67&1.36&65.81&1.87 \\
&3&37.97&51.47&1.36&67.97&1.79 \\
&4&41.38&55.11&1.33&70.86&1.71 \\
\hline
\multirow{3}{*}{240}  &2&43.57&53.36&1.22&82.21&1.89 \\
&3&46.50&59.89&1.29&83.32&1.79 \\
&4&52.24&61.89&1.18&88.34&1.69 \\
\hline
\end{tabular}
\end{table}

\section{Conclusion}
This paper proposes the AIR solve matrix completion problems without knowing the prior in advance. AIR parameterizes the Laplacian matrix in DE and can adaptively learn the regularization according to different data at different training steps. 
In addition, we demonstrate that AIR can avoid the over-fitting issue.
AIR is a generic framework for solving inverse problems that simultaneously encode the self-similarity and low-rank prior. Some interesting properties of the proposed AIR deserve further theoretical analysis. Such as the momentum phenomenon, which calls for a more detailed dynamic analysis.


\appendix
\setcounter{assump}{0}
\setcounter{lem}{0}
\setcounter{prop}{0}
\setcounter{thm}{0}
\setcounter{cor}{0}

\section{A Brief Review of DMF Algorithm and Theory}
\label{app..IntroDMF}
\begin{assump}\label{assump..balance}
Factor matrices are balanced at the initialization, i.e.,
$$\left[{\mW^{[l+1]}}(0)\right]^{\top} \mW^{[l+1]}(0)=\mW^{[l]}(0) \left[{\mW^{[l]}}(0)\right]^{\top} ,\quad l=0,\ldots, L-2.$$
\end{assump}
Under this assumption, Arora et al. have studied the gradient flow of the non-regularized risk function $\gL_{\sY}$, which is governed by the following differential equation 
\begin{equation}
\dot{\mW}^{[l]}(t)=-\frac{\partial}{\partial \mW^{[l]}} \gL_{\sY}\left(\mX(t)\right), \quad t \geq 0, \quad l=0, \ldots, L-1,
\label{eq..DynamicsNonReg}
\end{equation}
where the empirical risk
$\gL_{\sY}$ can be any analytic function of $\mX(t)$.
According to the analyticity of $\gL_{\sY}$, $\mX(t)$ has the following singular value decomposition 
\[
\mX(t)=\mU(t) \mS(t) \left[\mV(t)\right]^\top,
\]
where $\mU(t) \in \sR^{m, \min \left\{m, n\right\}},
\mS(t) \in \sR^{\min \left\{m, n\right\},
\min \left\{m, n\right\}}$, and
$\mV(t) \in \sR^{\min \left\{m, n\right\},n}$ are analytic
functions of $t$; and for every $t$, the matrices $\mU(t)$ and
$\mV(t)$ have orthonormal columns, while $\mS(t)$ is
diagonal. 
We denote the diagonal entries of $\mS(t)$ by 
$\sigma_{1}(t), \ldots,$ $\sigma_{\min \left\{m, n\right\}}(t)$, which are the signed singular values of $\mX(t)$. The columns of $\mU(t)$ and $\mV(t)$,
denoted by $\mU_{1}(t), \ldots, \mU_{\min \left\{m, n\right\}}(t)$
and $\mV_{1}(t), \ldots, \mV_{\min \left\{m, n\right\}}(t),$
are the corresponding left and right singular vectors, respectively. We have the following result about the evolutionary dynamics of the singular values of $\mX(t)$.

\begin{prop}[{\cite[Theorem 3]{Arora2019ImplicitRI}}]\label{prop..SinuglarValue}
    Consider the dynamics of \EQREF{eq..DynamicsNonReg} with initial data satisfying Assumption \ref{assump..balance}. Then the dynamics of signed singular values $\sigma_k(t)$ of 
    matrix $\mX(t)$ is governed by the following equation
    \begin{equation}
    \label{eq..arora}
    \begin{aligned}
        \dot{\sigma}_k(t)=-L \left(\sigma_k^{2}(t)\right)^{1-\frac{1}{L}} 
    \left\langle\nabla_{\mX} \gL_{\sY}(\mX(t)), \mU_{k,:}(t) 
    \left[\mV_{k,:}(t)\right]^\top\right\rangle,\  k=1, \ldots, 
    \min \left\{m, n\right\}.
    \end{aligned}
    \end{equation}
\end{prop}
If the matrix factorization is non-degenerate,
i.e., the factorization has depth $L \geq 2,$ the singular values 
need not be signed (we may assume $\sigma_k(t) \geq 0$ for all $t$).

Arora et al. claim that the term $\left(\sigma_k^2(t)\right)^{1-\frac{1}{L}}$ enhances the dynamics of large singular values while diminishing that of small ones. The enhancement/diminishment becomes more significant as $L$ increases \cite{Arora2019ImplicitRI}. As the $\gL_{\sY}$ can be any analyticity empirical risk function, we can replace $\gL_{\sY}$ with arbitrary analyticity function $\gL$ in \ref{eq..arora}.

\section{Proof of Theorem 2}
\label{app..Theorem2}
\begin{prop}\label{prop..RegGrad}
We have 
$$\nabla_{\mW}\left(\gR_{\mW}\left(\mX\right)\right)=2\mC\odot \mA-2\mathrm{tr}\left(\mC \mA'\right)\mA',$$
where $\mA'=\frac{\exp(\mW^{\top})}{\mathbf{1}_m^\top \exp(\mW)\mathbf{1}_m}$, $\mA=\mA'+{\mA'}^{\top}$, and
$$
\mC=\left(\gT\left(\mX\right){\gT\left(\mX\right)}^{\top}\odot \mI_{m}\right)\mathbf{1}_{m\times m}-\gT\left(\mX\right){\gT\left(\mX\right)}^{\top}.
$$
\end{prop}
\begin{proof}
     We denote $\mX=\gT\left(\mX\right)\in\sR^{m\times n}$, then we consider the differential of $\mathrm{tr}\left(\mX^\top \mL\mX\right)$, i.e., $d\left[\mathrm{tr}\left(\mX^\top \mL\mX\right)\right]$, note that
     $$
     \begin{aligned}
         &d\left[\mathrm{tr}\left(\mX^\top \mL\mX\right)\right]=\mathrm{tr}\left[d\left(\mL\mX\mX^\top\right)\right]\\
         = &\mathrm{tr}\left[\left(d\mA\mathbf{1}_{m\times m}\right)\odot I_m  \mX\mX^\top- d\mA\mX\mX^\top\right]\\
         = & \mathrm{tr}\left[\mX\mX^\top \left(\mI_{m}\odot \left(d\mA \mathbf{1}_{m\times m}\right)\right)-\mX\mX^\top d\mA\right]\\
         = & \mathrm{tr}\left[ \left(\mX\mX^\top \odot \mI_{m}\right)^\top d\mA \mathbf{1}_{m\times m}-\mX\mX^\top d\mA\right]\\
         = &  \mathrm{tr}\left[ \left(\left(\mX\mX^\top \odot \mI_{m}\right)\mathbf{1}_{m\times m}\right)^\top d\mA -\mX\mX^\top d\mA\right]\\
         = & \mathrm{tr}\left[\left(\mathbf{1}_{m\times m}\left(\mX\mX^\top \odot \mI_{m}\right)   -\mX\mX^\top\right)d\mA\right]\\
         = & \mathrm{tr}\left[\left(\left(\mX\mX^\top \odot \mI_{m}\right)  \mathbf{1}_{m\times m} -\mX\mX^\top\right)d\mA\right].
     \end{aligned}
     $$
     We denote $\mC = \left(\mX\mX^\top\odot \mI_{m}\right)\mathbf{1}_{m\times m}-\mX\mX^\top, S=\mathbf{1}_{m}^\top \exp\left(\mW\right)  \mathbf{1}_{m}$, then
     $$
     \begin{aligned}
         &d\left[\mathrm{tr}\left(\mX^\top \mL\mX\right)\right]  = \mathrm{tr}\left(\mC d\mA\right)\\
         = & \frac{1}{S^2}\mathrm{tr}\left[\mC \left(S \exp(\mW+\mW^\top)\odot d(\mW+\mW^\top)\right)\right.\\
         &\left.-\mC\left(\mathbf{1}_{m}^\top \left(\exp(\mW)\odot d\mW\right)\mathbf{1}_{m}\right)\exp(\mW+\mW^\top)\right]\\
         = & \mathrm{tr}\left[\mC \left(\mA\odot d \left(\mW+\mW^\top\right)\right)\right]\\
         &-\frac{1}{S^2}\mathrm{tr}\left[\mathbf{1}_{m\times m}\left(\exp(\mW)\odot d\mW\right)\right] \mathrm{tr}\left[\mC \exp\left(\mW+\mW^\top\right)\right]\\
         = & \mathrm{tr}\left[\mC \left(\mA\odot d \left(\mW+\mW^{\top}\right)\right)\right]-\frac{1}{S}\mathrm{tr}\left(\exp(\mW^T) d\mW\right) \mathrm{tr}\left(\mC \mA\right) \\
         = & \mathrm{tr}\left[\mC \left(\mA\odot d \left(\mW+\mW^\top\right)\right)\right]
         -\mathrm{tr}\left[\mathrm{tr}\left(\mC \mA\right)\mA'd\mW \right]\\
         = & \mathrm{tr}\left[\left(\left(\mC^\top \odot \mA\right)^\top+\mC^\top \odot\mA-\mathrm{tr}\left(\mC \mA\right)\mA'\right)d\mW\right].
     \end{aligned}
     $$
     Therefore,
     $$
     \begin{aligned}
         \nabla_{\mW} \mathrm{tr}\left(\mX^\top \mL_i\mX\right)&=\left(\mC^\top \odot \mA\right)^\top+\mC^\top \odot\mA-\mathrm{tr}\left(\mC \mA\right)\mA'\\
         &=2\mC \odot \mA-2\mathrm{tr}\left(\mC \mA'\right)\mA'.
     \end{aligned}
     $$
     Notice that $\mX=\gT\left(\mX\right)\in\sR^{m\times n}$, the proposition is proved.
\end{proof}

\begin{proof}(Proof of Theorem \ref{thm..ConvReg})
We rewrite the gradient in Proposition \ref{prop..RegGrad} into the following element wise formulation:
$$
\dot{\mW}_{kl}\left(t\right)=\left(2Ca(t)-4\mC_{kl}\right) \mA_{kl}'(t),
$$
where $Ca(t) = \mathrm{tr}\left(\mC(t) \mA'(t)\right)$ and the sub-index denotes the corresponding entry of the matrix. 

With the assumption that $\Big\|\gT\left(\mX\right)_{k,:}\Big\|_F^2=1$ and $\gT\left(\mX\right)_{kl}>0$, we have 
$$
0\leq \left(\gT\left(\mX\right)\gT\left(\mX\right)^\top\right)_{kl}\leq 1.
$$
Therefore $\mC_{kl}=\left(\gT\left(\mX\right)\gT\left(\mX\right)^\top\right)_{kk}-\left(\gT\left(\mX\right)\gT\left(\mX\right)^\top\right)_{kl}\geq 0$ and 
$\mC_{kk}=0$ as $\mA_{kl}'=S^{-1}\exp({\mW}_{kl})>0$, thus we have
$$
Ca(t)=\mathrm{tr}\left(\mC(t) \mA'(t)\right)=\sum_{k=1,l=1}^{m} C_{kl}\mA_{kl}'(t)\geq 0.
$$
Denote $\mC_{\hat{k}\hat{l}}\in\min\limits_{k, l} \mC_{kl}$, we have $\mC_{\hat{k}\hat{l}}\leq\mC_{kl}$ and then consider
$$
\begin{aligned}
    \dot{\mW}_{\hat{k}\hat{l}}(t)-\dot{\mW}_{kl}(t)
= 2Ca(t)\left(\mA_{\hat{k}\hat{l}}'(t)-\mA_{kl}'(t)\right)-4\left(\mC_{\hat{k}\hat{l}}\mA_{\hat{k}\hat{l}}'(t)-\mC_{kl}\mA_{kl}'(t)\right).
\end{aligned}
$$
As we initialize $\mW\left(0\right)=\varepsilon \mathbf{1}_{m\times m}$, thus $\mA_{kl}'(0)=\frac{1}{m^2}$, for $\forall k,l$. Therefore,
$$\dot{\mW}_{\hat{k}\hat{l}}(0)-\dot{\mW}_{kl}(0)=-4\left(\mC_{\hat{k}\hat{l}}-\mC_{kl}\right)\mA_{kl}'(0)=-\frac{4}{m^2}\left(\mC_{\hat{k}\hat{l}}-\mC_{kl}\right)\geq 0.$$ 
Then we have 
$\mW_{\hat{k}\hat{l}}(t)\geq \mW_{kl}(t)$ and $\mA_{\hat{k}\hat{l}}'(t)\geq \mA_{kl}'(t)$, the equality holds if and only if $t=0$ or $\mC_{\hat{k}\hat{l}}=\mC_{kl}$. Furthermore, $\dot{\mW}_{\hat{k}\hat{l}}(t)-\dot{\mW}_{kl}(t)\geq -4\left(\mC_{\hat{k}\hat{l}}-\mC_{kl}\right)\mA_{kl}'(0)$, then $\mW_{\hat{k}\hat{l}}(t)-\mW_{kl}(t)\geq  \mD_{\hat{k},\hat{l},k,l} t$, where $\mD_{\hat{k},\hat{l},k,l}=-4\left(\mC_{\hat{k}\hat{l}}-\mC_{kl}\right)\mA_{kl}'(0)\geq 0$. Next, we consider 
$$
\begin{aligned}
    \mA_{\hat{l}\hat{k}}'\left(t\right)
    &=\frac{\exp(\mW_{\hat{k}\hat{l}})}{\sum\limits_{k,l}\exp(\mW_{kl})}\\
    &=\frac{1}{\sum\limits_{k,l}\exp(\mW_{kl}-\mW_{\hat{k}\hat{l}})}\geq \frac{1}{\sum\limits_{k,l}\exp\left(-\mD_{\hat{k},\hat{l},k,l} t\right)}.
\end{aligned}
$$
As $\mC_{kk}=0$ and $\mC_{kl}\geq 0$, therefore $\mC_{kk}\in\min\limits_{kl}\mC_{kl}$=0. It is not difficult to show that $\mC_{\hat{k}\hat{l}}=0$ if and only if ${\gT\left(\mX\right)}_{\hat{k},:}={\gT\left(\mX\right)}_{\hat{l},:}$. We further simplify our notations by introducing the following notations
$$
\sS_1 = \left\{(k,l)\mid k\neq l,\gT\left(\mX\right)_{k,:}\neq \gT\left(\mX\right)_{l,:}\right\},
\sS_2=\left\{(k,l)\mid k\neq l,\gT\left(\mX\right)_{k,:}= \gT\left(\mX\right)_{l,:}\right\},
$$ 

$$
\sS_3=\left\{(k,l)\mid k=l\right\}.
$$
If we denote $\sS_2\cup \sS_3 = \left\{\left(k,l\right)\mid\mC_{kl}=0\right\}$ and $\left|\sS_2\cup \sS_3\right|=m+2s$, then when $\mC_{\hat{k}\hat{l}}=0$, $\mA_{\hat{l}\hat{k}}'\left(t\right)\geq \frac{1}{m+2s+\mE_{\hat{k}\hat{l}}(t)}$, where $\mE_{\hat{k}\hat{l}}(t)=\sum_{(k,l)\in\sS_1}exp(-\mD_{\hat{k},\hat{l},k,l}t)$ and $-\mD_{\hat{k},\hat{l},k,l}>0,\forall (k,l)\in\sS_1$ leads to $\mE_{\hat{k}\hat{l}}\left(+\infty\right)=0$. Notice that $\sum\limits_{k,l}\mA_{kl}'\left(t\right)=1$, we have
$$
\frac{1}{m+2s+\mE_{\hat{k}\hat{l}}(t)}\leq \mA_{\hat{l}\hat{k}}'\left(t\right)\leq \frac{1}{m+2s}.
$$
Therefore, 
$$\mA_{kl}'\left(+\infty\right)=\left\{\begin{array}{cc}
0& , \gT\left(\mX\right)_{k,:}\neq \gT\left(\mX\right)_{l,:} \\
        \frac{1}{m+2s} &, \gT\left(\mX\right)_{k,:} = \gT\left(\mX\right)_{l,:}
\end{array}\right. .$$ 
Furthermore, 
$$
\mA_{kl}\left(+\infty\right)=\mA_{kl}'\left(+\infty\right)+\mA_{kl}'\left(+\infty\right)=2\mA_{kl}'\left(+\infty\right)=2\mA_{kl}'\left(+\infty\right)=\frac{2}{m+2s}=\gamma,
$$ 
$$\mA_{kl}\left(+\infty\right)=\left\{\begin{array}{cc}
        0 &, \gT\left(\mX\right)_{k,:}\neq \gT\left(\mX\right)_{l,:} \\
        \gamma &, \gT\left(\mX\right)_{k,:} = \gT\left(\mX\right)_{l,:}
    \end{array}\right..$$ 
According to the definition of $\mL_i$,
$$
\mL_{kl}^*=\mL_{kl}(+\infty)=\left\{\begin{array}{cc}
    0 &  (k,l)\in\sS_1\\
    \gamma &  (k,l)\in\sS_2\\
    -\sum_{l'=1,l'\neq k}^{m} \mL_{kl'}^*& k=l 
\end{array}\right..
$$
Until now, we have proven that the adaptive regularization part of AIR will converge at the end of the training, which gives an upper bound of $\left|\mL_{kl}^*-\mL_{kl}(t)\right|$.
Next, we will focus on the convergence rate of AIR. We separately discuss the rate under three cases in the formulation above.

If $(k,l)\in\sS_2$ or $k=l$, we denote $D=\min D_{kl}$, according to the definition of $\mE_{kl}(t)$, we have $\mE_{kl}(t)\leq \exp(-D t)$ and
$$
\begin{aligned}
\left|\mA_{kl}^*-\mA_{kl}(t)\right|
&=2\left[\frac{1}{m+2s}-\frac{1}{m+2s+\mE_{kl}(t)}\right]\\
&\leq 2\frac{\mE_{kl}(t)}{\left(m+2s\right)^2}\leq 2\frac{\frac{m\left(m-1\right)}{2} \exp(-D t)}{\left(m+2s\right)^2}\leq \exp\left(-D t\right).
\end{aligned}
$$
In particular, when $(k,l)\in\sS_2$ we have
$$
\left|\mL_{kl}^*-\mL_{kl}(t)\right|=\left|\mA_{kl}^*-\mA_{kl}(t)\right|\leq \exp\left(-D t\right). 
$$

If $(k,l)\in\sS_1$,
$$
\begin{aligned}
\left|\mL_{kl}^*-\mL_{kl}(t)\right|
&=\left|\mA_{kl}^*-\mA_{kl}(t)\right|
=\left|\mA_{kl}(t)\right|\\
&\leq \left|\sum_{k',l'\in\sS_1\cup \sS_3}\mA_{k'l'}(t)\right|\\
&=\left|2-\sum_{k',l'\in\sS_2\cup \sS_3}\mA_{k'l'}(t)\right|\\
&=\left|\gamma \left(m+2s\right)-\sum_{k',l'\in\sS_2\cup \sS_3}\mA_{k'l'}(t)\right|\\
&=\left|\sum_{k',l'\in\sS_2\cup \sS_3}\left(\gamma-\mA_{k'l'}(t)\right)\right|\\
&\leq\sum_{k',l'\in\sS_2\cup \sS_3}\left|\gamma-\mA_{k'l'}(t)\right|\\
&=\sum_{k',l'\in\sS_2\cup \sS_3}\left|\mA_{k'l'}^*-\mA_{k'l'}(t)\right|\leq 2\exp(-D t)/\gamma. \\
\end{aligned}
$$

If $k=l$,
$$
\left|\mL_{kl}^*-\mL_{kl}(t)\right|
=\left|\sum_{l'=1,l'\neq k}^{m}\left(\mL_{kl'}^*-\mL_{kl'}(t)\right)\right|\leq  2\left(m-1\right)\exp(-D t)/\gamma .
$$
\end{proof}

\section{Deeper and wider are beneficial}
\label{sec..deepwidth}
We will illustrate that the multi-layer matrix factorization enjoys implicit low-rank regularization without any special requirement on the initialization and constraint on shared dimension $r$ (width of the matrix), which overcomes the sensitivity on initialization in the two-layer matrix factorization \cite{albright2006algorithms}. That is, we do not need to estimate $r$ in advance. We design this experiment to explore the effects of $L$ and $r$ in \Figref{fig:layer_width}, and we see that the models perform better with bigger $L$ and $r$. However, bigger $L$ and $r$ also increase the computation complexity. Thus, we use $L=3$ and $r=\min\left\{m,n\right\}$ as our factorized model.

\begin{figure}
    \begin{center}
        \begin{tabular}{cc}
           \includegraphics[width=0.43\linewidth]{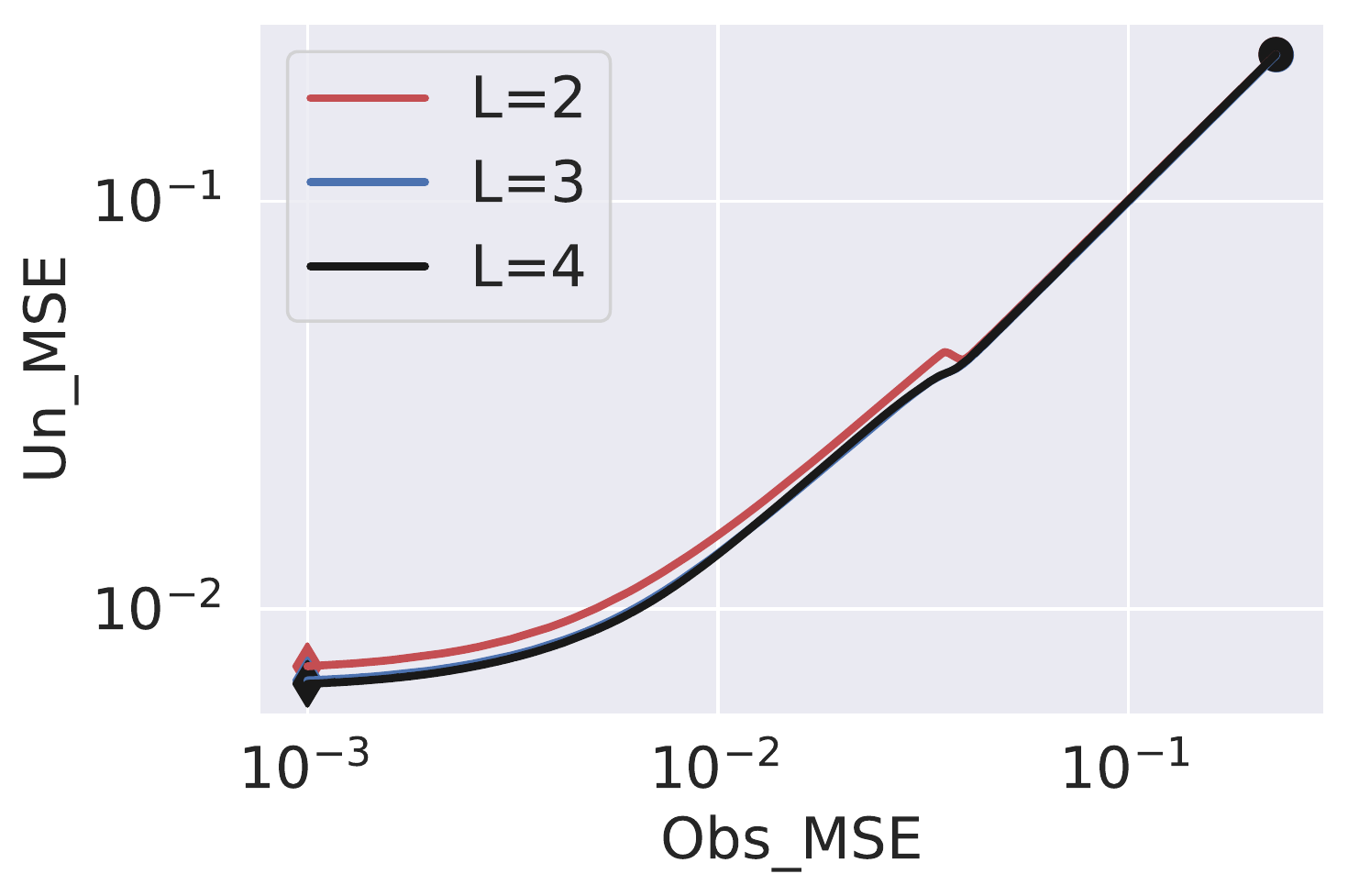}  & \includegraphics[width=0.43\linewidth]{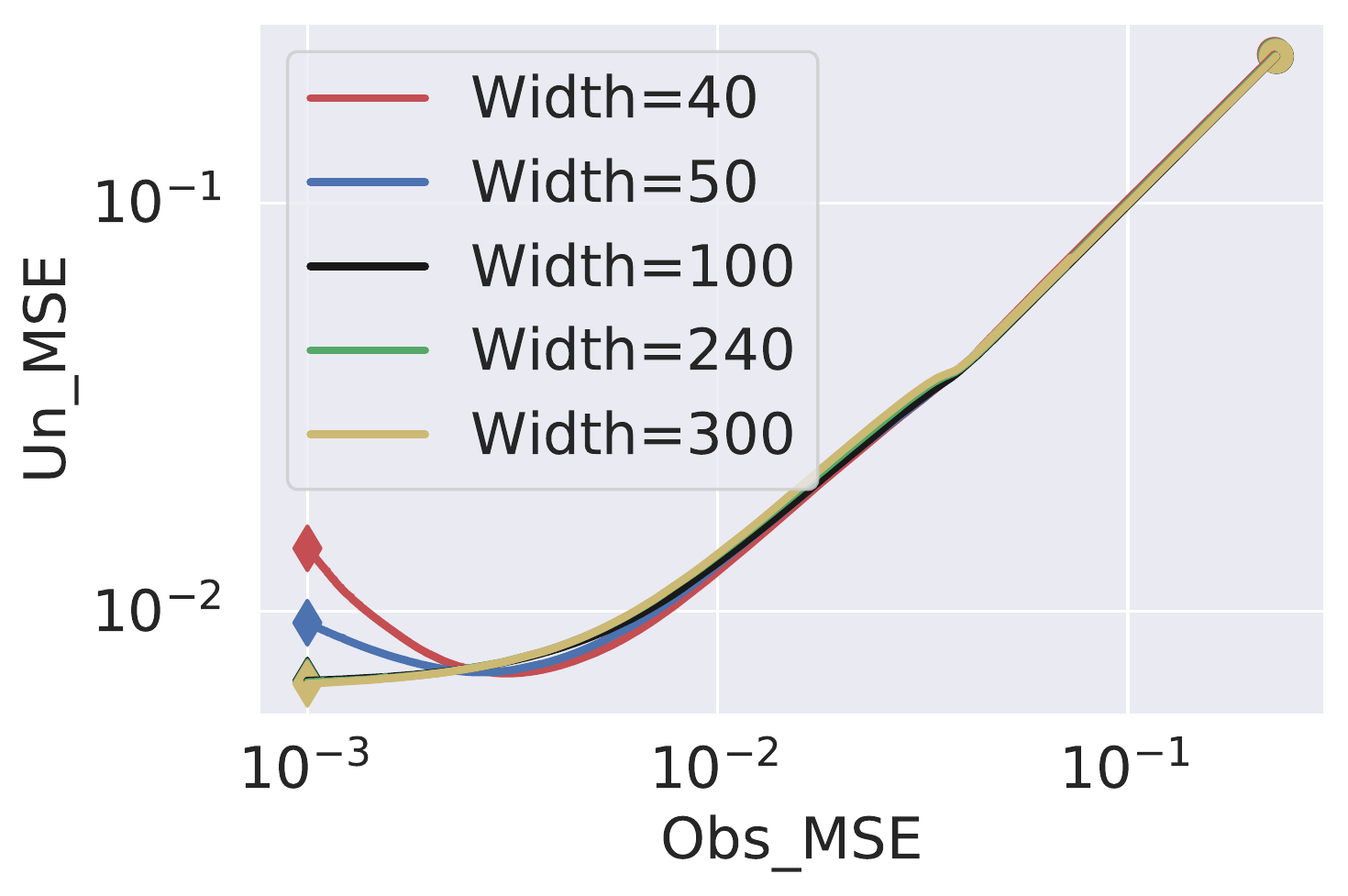} \\
            (a) Different numbers of the factorized matrix. & (b) Different shared dimensions.
        \end{tabular}
    \end{center}
    \caption{Trajectories of training AIR for inpainting the Barbara image with randomly missing $50\%$ pixels with (a) different number of factorized matrix, and (b) different shared common dimension.}
    \label{fig:layer_width}
\end{figure}

\bibliographystyle{unsrt}
\bibliography{references}

\end{document}